\documentclass[letterpaper,10pt]{article}

\usepackage{probml}

\ShortHeadings{RAMP: Recognition parametrisation by Amortised Message Passing}
\usepackage{blindtext}
\usepackage{wrapfig}

\usepackage{xcolor}
\usepackage{subcaption}
\usepackage{mathtools}
\usepackage{amsmath}
\usepackage{amsthm}
\usepackage{amssymb}
\usepackage{import}
\usepackage{bm}
\usepackage{tikz}
\usetikzlibrary{positioning, arrows.meta, fit}
\tikzoption{fit to node}{\tikzset{shift=(#1.south west), x=(#1.south east), y=(#1.north west)}}
\usepackage{xspace}
\usepackage{microtype}
\usepackage{graphicx}
\usepackage{booktabs} % for professional tables
\usepackage{makecell}
\usepackage{multirow}
\newcommand\eqconst{\overset{+c}{=}}
\usepackage{wrapfig}
\usepackage[percent]{overpic}

\usepackage{algorithm,algorithmic}
\usepackage{placeins}

\usepackage{xr}
\newtheorem{assumption}[theorem]{Assumption}
\newcommand{\eqlft}{(}
\newcommand{\eqrgt}{)}
\creflabelformat{equation}{\eqlft#2#1#3\eqrgt}

\newcommand{\crefalt}[1]{%
\begingroup%
\let\eqlft\relax%
\let\eqrgt\relax%
\cref{#1}%
\endgroup%
}

\usepackage{xcolor}
\usepackage{listings}
\definecolor{codegreen}{rgb}{0,0.6,0}
\definecolor{codegray}{rgb}{0.5,0.5,0.5}
\definecolor{codepurple}{rgb}{0.58,0,0.82}
\definecolor{backcolour}{rgb}{0.95,0.95,0.92}

\lstdefinestyle{mystyle}{
	backgroundcolor=\color{backcolour},
        commentstyle=\color{codegreen},
	keywordstyle=\color{magenta},
	numberstyle=\tiny\color{codegray},
	stringstyle=\color{codepurple},
	basicstyle=\ttfamily\footnotesize,
	breakatwhitespace=false,         
	breaklines=true,                 
	captionpos=b,                    
	keepspaces=true,                 
	numbers=left,                    
	numbersep=5pt,                  
	showspaces=false,                
	showstringspaces=false,
	showtabs=false,                  
	tabsize=2
}

\usepackage{mbubble}
\mbubblestyle{ms}{color=blue,font={\tiny},label=ms,hlcolor=blue!15}

\input{m.symbols}

\newcommand{\rampprior}[1]{f^{#1}_0\left(z_{#1}\right)}
\newcommand{\rampf}[2]{f^{#1}_{#2}(z_{#1}|\setX^{#1}_{#2})}
\newcommand{\rampfzx}[2]{f^{#1}_{#2}(z_{#1}|x_{#2})}
\newcommand{\rampff}[2]{f^{#1}_{#2}}
\newcommand{\rampfn}[3]{f^{#1}_{#2}\left(z_{#1}|{\setX^{#1}_{#2}}^{(#3)}\right)}
\newcommand{\rampemp}[2]{p_{\text{emp}}\left(\setX^{#1}_{#2}\right)}
\newcommand{\rampempname}{p_{\text{emp}}}
\newcommand{\rampF}[2]{F^{#1}_{#2}\left(z_{#1}\right)}

\newcommand{\rpmf}[1][p]{f_{#1}\left(\Z|x_{#1}\right)}
\newcommand{\rpmfn}[2]{f_{#1}\left(\Z|x_{#1}\nn[#2]\right)}
\newcommand{\rpmemp}[1][p]{p_{\text{emp}}\left(x_{#1}\right)}
\newcommand{\rpmempname}{\ensuremath{p_{\text{emp}}}\xspace}
\newcommand{\rpmF}[1][p]{F_{#1}\left(\Z\right)}

\newcommand{\eff}{\mathcal{F}}

\newcommand\nn[1][n]{^{(#1)}}

\newcommand\XN[1][N]{\ensuremath{\mathbb{X}\nn[#1]}}

\newcommand\T{\setT}
\newcommand\G{\setG}
\newcommand\E{\setE}
\newcommand\Ezz{\setE_{\text{zz}}}
\newcommand\Ezx{\setE_{\text{zx}}}
\newcommand\V{\setV}
\newcommand\X{\setX}
\newcommand\Z{\setZ}

\DeclareDotOptCommand\Tkj{2}.{k}{j}{\T^{#1}_{#2}}
\DeclareDotOptCommand\Xkj{2}.{k}{j}{\X^{#1}_{#2}}

\DeclareDotOptCommand\msg{2}.{}{}{_{{#1}\to{#2}}}
\DeclareDotOptCommand\bck{2}.{}{}{_{{#1}\leftarrow{#2}}}
\DeclareDotOptCommand\alphakj{2}.{k}{j}{\alpha\msg.{#2}{#1}(z_{#1})}

\DeclareDotOptCommand\Ne{1}.{}{\partial_{{#1}}}
\DeclareDotOptCommand\Nebar{1}.{}{\text{N\=e}\optly{#1}.{(#1)}.}

\newcommand{\subfigref}[2]{Figure~\hyperref[#1]{\ref*{#1}#2}}

\newcommand{\setE}{{\mathcal{E}}}

\newcommand{\setG}{{\mathcal{G}}}

\newcommand{\setT}{{\mathcal{T}}}
\newcommand{\setU}{{\mathcal{U}}}
\newcommand{\setV}{{\mathcal{V}}}

\newcommand{\setX}{{\mathcal{X}}}

\newcommand{\setZ}{{\mathcal{Z}}}

\begin{document}

\title{
RAMP: Recognition parametrisation by\\ Amortised Message Passing
}

\author[1]{Lior Fox}
\author[1,2]{Kai Biegun}
\author[1]{James Heald}
\author[1]{Samo Hromadka}
\author[1]{Arielle Rosinski}
\author[1]{Maneesh Sahani}

% \author[1,$*$,$\dagger$]{First Author}
% \author[1,2,$*$]{Second Author}
% \author[2]{Third Author}

\affil[1]{Gatsby Computational Neuroscience Unit}
\affil[2]{Centre for Artificial Intelligence}
\affil[ ]{University College London}
% \affil[$*$]{Equal contribution}
% \affil[$\dagger$]{Correspondence to \url{email@example.com}}

\maketitle

\begin{abstract}
  A central aim of unsupervised learning is to uncover latent factors that explain dependencies
  among observations.  Probabilistic models typically achieve this by introducing multiple latent
  variables linked through a graph of conditional relationships, with distributional parameters
  and their dependence learnt from data.  Learning relies either on distributional choices that
  allow tractable belief propagation, or on approximations that scale poorly with model size and
  complexity. We build on the recently developed recognition-parametrised modelling paradigm to
  propose an alternative approach: RAMP, a method that implicitly defines latent structure by
  learning a flexible, nonlinear, amortised message-passing framework.  We show that RAMP enables
  efficient likelihood-based recovery of latent-variable distributions within expressive nonlinear
  models acting on complex high-dimensional data. 
\end{abstract}

% \lior{
% TODO camera ready:\\
% - Structural mis-specification experiments: appendix? and ref from main text \\
% - Non-linear trees: expand current appendix. perhaps train the BBVE baseline on the entire $(\beta,p)$ sweep created for RAMP \\
% - EP baseline for the non-linear tree sweep? \\
% - Pay our tribute and add some reference the reviewer had asked for
% }

\section{Introduction}
We understand the world by building a mental model of it.
Good models help to describe complex phenomena simply, by identifying and recovering relevant latent factors that explain patterns of regularity and dependence amongst observations.
% This filters the ``signal'' (i.e., aspects directly captured by the model) from the ``noise'' (all other aspects of the observations).
% This signal might be a caricature of the full complexity observed in the real-world phenomenon, but often admits a much simpler mechanistic description. 
% The art of modelling is in balancing this trade-off. %for the relevant problem at hand.
%
Mechanising such model building has been an overarching goal for artificial intelligence and machine learning for decades.
Probabilistic methods are particularly well-suited to this task.
They express non-deterministic links between variables, represent uncertainty in a mathematically coherent way, and offer a principled approach for inference in the form of the posterior distribution over hidden variables, given the observations.

Yet, in the tradeoff between expressivity and simplicity, probabilistic methods often lean to the latter.  
The challenge is that computations in probabilistic models depend on integration over variables, and these integrals are only tractable for the simplest of distributional choices.  
Numerical or approximate methods have been developed in many other settings, but these become increasingly forbidding as the complexity of probabilistic dependence increases.
%
% This is because the inference problem is fundamentally hard: 
% in general, many different combinations of the latent variables could give rise to the same observation, and in order to find the posterior it is necessary to consider all of them.
% Unless some strong distributional assumptions are added, exact inference in even slightly complicated models is therefore intractable, and some approximations have to be made.
%
% % % % Maybe revert to this line of arguemnt.
% % % % It might be too ambitious to cover in 2-3 sentences though.
% At the other end of the spectrum, pattern-recognition methods learn to directly approximate a complicated target function with a set of expressive, adaptive components. 
%
This need to approximate has been a major driving force behind attempts to integrate neural network approaches, which excel at learning to approximate complicated functions, into probabilistic methods. 
Despite some success, a general unified approach is still missing for this integration.
One issue is that the nature of approximations offered by the two approaches is not always compatible, and the internal representations generated by generic deep learning methods do not easily lend themselves to probabilistic interpretation.

This paper presents a novel approach towards merging adaptive neural-network components into a probabilistic framework, based on three core ideas.
The first is to train a set of networks to collectively perform inference, leveraging the ability of pattern-recognition methods to amortise complicated transformations.
The second is to constrain the way in which the outputs of these networks are interpreted, transformed, and combined together.
These constraints, together with the learning objective itself, are derived directly from probabilistic considerations encoded in a graphical model.
Finally, the third core idea is that of recognition-parametrisation, allowing the \textit{inference} (``recognition'') procedure to directly define the model itself.

\begin{figure}
    \centering
    \includegraphics[width=0.4\textwidth,page=2]{figures/tikz_figs/tree_conditional_independence.pdf}%
    \includegraphics[width=0.4\textwidth,page=1]{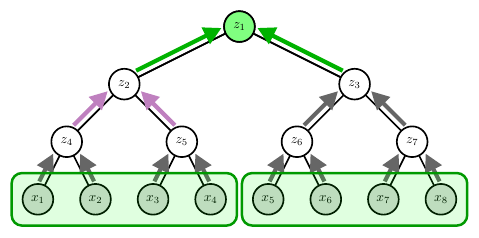}%
    \caption{
        RAMP works by learning a system of amortised message-passing transformations (directed arrows) within a tree structured graphical model. At each latent node, the incoming messages (bold coloured arrows) serve as a set of \textit{recognition factors}, that are trained to model the specific conditional independence relationships among observations induced by the corresponding latent (denoted by the coloured bounding boxes). The use of a single set of message-passing functions, shared between all the nodewise RPMs, leads to every message reporting a sufficient statistic for the set of its incoming observations, relative to the rest of the observations.
    }
    \label{fig:ramp_arch}
\end{figure}

      %   \lior{complete}
      % Overview of RAMP. In a tree structured graphical model with latent and observed variables,
      % different latent variables account for different patterns of conditional
      % independences among observations (highlighted boxes on left and right). RAMP leverages these conditional independence relationships to define a nodewise RPM at each latent node, parameterised by a set of recognition factors. 
%
    % A tree stuctured graphical model with latent ($\setZ$) and observed ($\setX$) variables.
    % Different latent variables (highlighted $z_i$ and $z_j$, left and right) account for different patterns of conditional independence among observations.
    % At each latent variable, a nodewise model for $p(z_k,\setX)$ is parameterised by a set of recognition factors (incoming arrows into $z_i$,$z_j$).

    % Two latent variables that capture different patterns of conditional independence implied by a common graphical model.
    % A nodewise-model is constructed at each latent, parameterised by a set of recognition factors (incoming arrows into $z_i$,$z_j$).
    %
    %
    %
    %
    %  %
    %
    %
    %
    %
    %
    % Two nodewise models capturing different patterns of conditional independence implied by common graphical model.  Each model is parametrised by a set of recognition factors (incoming arrows into $z_i$,$z_j$).
    % %
    % Recognition factors are parametrised with neural networks across each edge, and follow a message-passing structure, reusing elements between models (green arrows in the left figure and purple arrows in the right).

\section{Mathematical Background}

% We briefly review some necessary mathematical background.  Further details appear in the Appendix.

\subsection{Recognition-Parametrised Models}

The Recognition-Parametrised Model \citep[RPM;][]{RPM} describes $P$ observed variables 
$\X = \{x_1 \dots x_P\}$ in terms of one or more latents $\Z = \{z_1 \dots z_K\}$, 
such that the dependence amongst the observations is captured by $\Z$.  Such a model would 
conventionally be defined by parametrising a prior $p(\Z)$ and generative 
conditionals $p(x_p | \Z)$.  The RPM is semi-parametric, combining a 
nonparametric summary of the $x_p$-marginals, $\rpmemp$, with normalised recognition factors $f(\Z|x_p)$
that map observed values to posterior belief distributions over the latent variables.  The joint
is then
\begin{align}\label{eq:simple-rpm}
  p(\Z,\X) &= p(\Z) \prod_p \frac {\rpmf}{\rpmF} \rpmemp &
  &\text{with}&
  \rpmF &= \intdx[dx_p] \rpmf \rpmemp \,.
\end{align}
Like Walker et al., we take $\rpmemp = \frac 1N \sum_{n=1}^N \delta(x_p - x_p\nn)$, 
so that $\rpmF = \frac 1N \sum_{n=1}^N \rpmfn pn$.

The RPM is a normalised model with a proper likelihood function.  Recognition parameters can be learnt by optimising a free energy bound
\citep{neal+hinton:1998} as discussed in \cref{sec:supp-free-energies}.

\subsection{Tree-Structured Graphical Models}
We consider undirected, tree-structured, graphical models.
To simplify presentation, we focus on models with pairwise (and singleton) factors, but the key ideas we develop generalise to any tree-structured factor graph (see \cref{subsec:ramp_general_factor_graphs}).

Let $\T = (\V,\E)$ be a tree with nodes $\V$ and (undirected) edges $\E$.
We identify the nodes with variables in a probabilistic model, partitioned into \textit{latent variables} $\Z$ and \textit{observed variables} $\X$; $\V=\Z\cup\X$,
and refer to the nodes interchangeably either by the variable index (e.g.\ $k\in \{1\dots K\}$) or by the associated variable label (e.g.\ $z_k$).
Without loss of generality%
\footnote{If an observation were at an internal node it would partition $\T$ into two or more subtrees within which inference and learning would proceed independently.}%
, we take the observations to be leaf nodes of the tree; implying that $\E$ is partitioned into sets linking two latents ($\Ezz$), or a latent to an observation ($\Ezx$); $\E=\Ezz\cup\Ezx$.  
This graphical structure defines a family of joint probability distributions over observed and latent variables, satisfying the following factorisation:
\begin{equation}\label{eq:tree-factors}
    p(z_{1:K},x_{1:P}) = \frac{1}{Z}\prod_k\phi_k(z_k)\prod_p\psi_p(x_p)
    \times
    \prod_{\mathclap{(k,k')\in\setE_{\text{zz}}}}\Phi_{kk'}(z_k,z_{k'})
    \prod_{\mathclap{(k,p)\in\setE_{\text{zx}}}}\Psi_{kp}(z_k,x_p)
\end{equation}

\subsubsection{Trees and Conditional Independence}\label{sec:leveraging-ci}
In a tree-structured graph, the removal of a latent node $z_{k}$ with degree $\deg(k)$ partitions the graph $\T$ into a collection of $\deg(k)$ disjoint sub-trees, one for each edge at $k$ (see \cref{fig:ramp_arch}).  
We label the sub-trees by the neighbour to which the corresponding edge connects: $\lbrace \Tkj \rbrace_{j\in\Ne.k}$, where $\Ne.k$ is the neighbourhood of $z_k$ (note that $j$ may index an observed variable, in which case $\Tkj$ contains the single node $x_j$).
As $\Tkj$ is defined by the removal of $z_k$ from the graph, $\Tkj\neq\Tkj.jk$.
 
Let $\Xkj$ be the set of observed variables in the tree $\Tkj$. 
The graphical model defined by $\T$ implies that the sets $\Xkj$ for $j\in\Ne.k$ are conditionally independent given $z_k$.
% \begin{equation*}
% p(z_k,\X)=p(z_{k})\prod_{j\in\Ne.k}p(\Xkj|z_{k})
% \end{equation*}
% \begin{equation*}
% p(z_k,\setX)=p(z_{k},\setX^k_{j_1},\dots,\setX^k_{j_{\deg(k)}})=p(z_{k})\prod_{j\in\Ne.k}p(\setX^k_{j}|z_{k})
% \end{equation*}
We refer to this property as the $z_k$-``nodewise'' conditional independence partition of the observations.
These conditional independence properties provide the principal signal for learning within a nonlinear tree-structured model: $z_k$ is defined by the need to uniquely model the dependence amongst the different subsets of observations $\{\Xkj\}$.
The induced partition depends on where the latent variable sits within the tree.
For example, in the graph shown in \cref{fig:ramp_arch}, $x_2$ is conditionally independent of $x_3$ given $z_2$, but not $z_1$.
From this point of view, a latent variable is only discoverable on the basis of the conditional independence it induces amongst observations---to paraphrase \citet{firth1957}, ``you shall know a latent by the observables it keeps''. 
%
% Consequently, two hypothesised latent variables that induce the same structure of conditional independences are indistinguishable from each other, without further strong distributional assumptions. 
%
Thus, we focus on trees in which each latent variable induces a unique partition of the observations.  This is captured by the following assumption:\\[-3.5ex]
\begin{assumption}\label{assumption:tree_structure}
    For every pair $z_i$, $z_j$ ($i\neq j$) in $\setT$, there exist $x,x'\in\setX$ that belong to the same sub-tree of $z_i$ but to different sub-trees of $z_j$.
\end{assumption}
\vspace{-2ex}
% Note that the distribution over variables within a given $\setX^k_j$ might be \textit{further} factorised, conditioned on other latents (but not on $z_{k}$ alone).

% \subsubsection{Nodewise Conditional Independences Imply Graph Structure}

We have focused on conditional independence induced on the observed (leaf) nodes, but independence amongst latents is an equally important aspect of the model.  
For example, in a Markovian state-space model, these independencies define the Markov chain structure of the latent series $z_t$. 
Fortunately, as the following lemma shows, a model defined by all nodewise conditional independence relationships on observations necessarily also captures conditional independence amongst the latents.
% The nodewise models of \cref{eq:local_rpm_zk} capture the conditional independence relationships between individual latent nodes and observations, 
%% but they do not \textit{directly} model conditional independencies amongst the different latent nodes. 
% %
% This latent independence structure is nonetheless an important aspect of the model, encoded in the original graphical structure. 
% %
% %
% In fact, as the following lemma shows, simultaneously satisfying all of the nodewise conditional independences is sufficient to ensure that all the latent-to-latent conditional indepencies of the original graph are satisfied as well.
%
\begin{lemma}\label{lemma:local_CIs_suff}
    Let $\T = (\V, \E_{\T})$ satisfy Assumption \ref{assumption:tree_structure}.
    Let $\G = (\V, \E_{\G})$ be a connected graph on the same nodes, satisfying the nodewise conditional independence properties implied by $\T$. Then, $\E_{\G}=\E_{\T}$.
\end{lemma}
\vspace{-2ex}
\begin{proof}
% Consider two latent variables $z_i$ and $z_j$, $i\neq j$.
% We first prove that there must be a path from $z_i$ to $z_j$ in $\setG$.
% Because the partition of observations induced by $z_i$ is not identical to the one induced by $z_j$, there must exist two observed variables $x_a,x_b$ such that $x_a\cind x_b|z_i$ and $x_a\notcind x_b|z_j$. 
% From the first property, $z_i$ separates all paths from $x_a$ to $x_b$. 
% From the second property, $x_a$ and $x_b$ are in the same sub-tree when $z_j$ is removed. 
% Hence the path from $x_a$ to $x_b$, which includes $z_i$ is within the same sub-tree, such that there is a path from $z_j$ to $z_i$.

%
% The graph to keep in mind is:
%      x_a -- zi -- x_b         {in G}
%   AND
%      [x_a ...  x_b] --- z_j   {in G}
%
%  which means there must be a path z_j --- z_i    {in G}
%
We show that $\E_\G \subseteq \E_\T$; equality follows from the assumptions that $\T$ is a tree and $\G$ is connected. 
Let $(i,j)\notin\E_{\T}$.
Because $\T$ is a tree with $\X$ at its leaves, there must be a path from $z_i \leadsto z_j$ passing through at least one additional node $z_k$ ($k\not\in\{i,j\}$).  
By Assumption \ref{assumption:tree_structure} there exist distinct $x_a, x_b$ such that $x_a \cind x_b | z_i$ but $x_a\notcind x_b | z_k$.  
Similarly, there are $x_c, x_d$ separated by $z_j$, but not by $z_k$.
As $\G$ satisfies the nodewise conditional independencies of $\T$, there must be a path $x_a\leadsto x_b \in \E_\G$ that passes through $z_i$ but not $z_k$. 
Similarly, there is a path $x_c \leadsto x_d\in\E_\G$ through $z_j$ but not $z_k$.
Now, by construction, $x_a$ and $x_d$ belong to different sub-trees partitioned by $z_k$, and so $x_a \cind x_d | z_k$ in $\T$, from which it follows that $x_a \cind x_d | z_k$ also in $\G$.
But if $\G$ contained the edge $(i,j)$, there would be a path $x_a \leadsto z_i \to z_j \leadsto x_d$ which does not pass through $z_k$, violating this separation.
So  $(i,j)\notin\E_{\T} \Rightarrow (i,j)\notin\E_\G$ and $\E_\G \subseteq\E_\T$ as required.
%     x_a -- z_i -- x_b
%     [x_a .. z_i .. x_b] -- z_k
%                                   x_c -- z_j -- x_d
%                            z_k - [x_c .. z_j .. x_d]
%
% In G, the path x_c - x_d goes through z_j but not z_k.
% In G, the path x_a - x_b goes through z_i but not z_k.
% G models the fact that x_a \cind x_d given z_k -- this is because x_a and x_d are in different sub-trees of z_k.
% => In G, there are paths x_a - z_i; x_d - z_j, not through z_k.
% If we were to add an edge zi-z_j, there would be a path x_a -- x_d not through z_k, contradicting the \cind property at z_k.
%
%
\end{proof}

\subsubsection{Message passing}

% Inference  in graphical models involves finding the marginal posteriors (sometimes pairwise across an edge) of the latents given an instance of the observed variables; 
% % The \textit{inference} problem, then, is to compute all the marginal posteriors given a particular instantiation of the observed variables 
% that is, computing $p(z_k|\X)$ for $k=1,\dots, K$.

It is well known that exact inference is tractable for tree-structured graphs \citep{koller+friedman:2009}, and can be implemented efficiently by belief propagation \citep[BP;][]{pearl:1988} \textit{provided} that integrals over local beliefs can be computed exactly.
Defining potentials $\alphakj = p(\Xkj|z_k)$ at each latent, BP provides the recursive updates
\begin{align}\label{eq:bp-messages}
  \alphakj &=\!\!\intdx[dz_j] p(z_j | z_k) \prod_{l\in\Ne.j\backslash k} \alphakj.jl
   \quad\text{if  $j \in \Z$;} 
   &
   \alphakj &= p(x_j | z_k)
   \quad\text{if  $j \in \X$,}
\end{align}
where $p(z_j|z_k)$ and $p(x_j|z_k)$ are derived from the factors of \cref{eq:tree-factors}.
%
% Consider a graphical model with observed variables $\X$, latent variables $\Z$ (i.e., nodes $\setV =
% \X\cup\Z$) and edge set $\E = \Ezz \cup \Ezx$.  The general form of the joint is
% \begin{equation*}
%     p(z_{1:K},x_{1:P}) = \frac{1}{Z}\prod_k\phi_k(z_k)\prod_p\psi_p(x_p)
%     \prod_{(k,k')\in\Ezz}\Phi_{kk'}(z_k,z_{k'})
%     \prod_{(k,p)\in\Ezx}\Psi_{kp}(z_k,x_p) \,,
% \end{equation*}
% where $\phi_k$, $\psi_k$, $\Phi_{kk'}$ and $\Psi_{kp}$ are (typically 
% parametrised) non-negative factors and $Z$ is a normalising constant.  If the
% graph is tree-structured, it is possible to write the same distribution in terms
% of its normalised singleton and pairwise marginals,
% \begin{equation*}
%     p(z_{1:K},x_{1:P}) = \prod_kp(z_k)\prod_pp(x_p)
%     \prod_{(k,k')\in\Ezz}\frac{p(z_k,z_{k'})}{p(z_k)p(z_{k'})}
%     \prod_{(k,p)\in\Ezx}\frac{p(z_k,x_p)}{p(z_k)p(x_p)} \,,
% \end{equation*}
% or, defining an arbitrary $z_k$ as the root node of the tree, in terms of directed
% conditionals,
% \begin{equation*}
%   p(z_{1:K},x_{1:P}) = p(z_k)
%   \prod_{k'\neq k} p(z_{k'} | \text{pa}(z_{k'}))
%     \prod_{p} p(x_p | \text{pa}(x_p)) \,.
% \end{equation*}
%
% Belief propagation \citep[BP;][]{pearl:1988} is an efficient message-passing algorithm
% to compute various quantities in the graph: these include singleton and pairwise
% marginals, posterior marginals and the likelihood.  
%
To compute $p(z_k | \X)$, we follow  the ``inward'' path from each $x_p$ to $z_k$, computing $\alphakj.sr$ for each edge $(r,s)$ encountered.  Then
\begin{equation}\label{eq:bp-posterior}
  p(z_k | \X) = \frac1{L_k} p(z_k,\X) = \frac1{L_k} p(z_k) \prod_{j\in\Ne.k} \alphakj\,,
\end{equation}
where $L_k$ is easily computed by integration over the single variable $z_k$.
Furthermore, $L_k = p(\X)$ is the model likelihood, and so, provided messages are computed exactly, is the same at every $k$.  As will be seen, RAMP optimises all of these likelihoods together.  Fortunately, BP allows efficient computation of every nodewise likelihood: a single ``inward-outward'' pass of messages with respect to any one node in the tree (taking order $|\setV|$ steps) is sufficient to compute them all.

In practice, exact messages can only be computed for a limited class of generative conditionals.  More general (nonlinear) dependence between variables requires approximation \citep{wainwright+jordan:2008}, often combined with Monte-Carlo or numerical integration \citep[e.g.,][]{bbvi2014}.  

\section{RAMP}\label{sec:model_definition}

We are now ready to define RAMP.  Let $\X$ be a set of observed variables, and $\{\X\nn\}_{n = 1}^N$ an observed dataset, with a joint captured by a tree-structured model $\T = (\Z\cup\X, \E)$ satisfying Assumption~\ref{assumption:tree_structure}.  

\subsection{Amortised Message-Passing Network}

We define the RAMP model in terms of message-passing-based recognition, in the spirit of the RPM (\cref{eq:simple-rpm}). 
The latent variables $z_k$ in \cref{eq:tree-factors} can be transformed by an arbitrary bijection, and the relevant factors composed with the inverse, without altering the joint distribution.  Thus, without loss of generality, we fix marginal priors $\rampprior k$.
Let $\rampf kj$ be a recognition factor that defines a normalised message representing a belief over $z_k$ based on $\Xkj$.  The tree-based factorisation implies that this belief should depend on messages to $z_j$ in a 
manner analogous to \cref{eq:bp-messages},
\begin{equation}\label{eq:message_passing_belief}
  \rampf kj = \intdx[dz_j] p(z_k|z_j) \rampprior{j} \prod_{l\in\Ne.j\backslash k} \frac{\rampff jl(z_j | \Xkj.jl)}{\rampprior{j}}\,.
\end{equation}
\cref{eq:message_passing_belief} requires parametrisation of $p(z_k|z_j)$ (or of the factors in \cref{eq:tree-factors} from which it is derived), and depends on evaluation of the integral over $z_j$.  Thus, graph factors are often constrained to a simple tractable form.
In RAMP, we instead amortise the belief-to-belief transformation directly, defining
\begin{equation}\label{eq:ramp_message_functional}
    \rampf kj = \setG\msg.jk\left(\rampprior{j}\prod_{l\in\Ne.j\backslash k}\frac{\rampf jl}{\rampprior{j}}\right),
\end{equation}
in terms of a generic functional $\setG\msg.jk$ for latent-to-latent messages.
%(where $\G\msg.jk$ and $\G\msg.kj$ are distinct functionals, and not generally inverses).  
%
Note that the tree structure implies that $\Xkj = \bigcup_{l\in\Ne.j\backslash k}\Xkj.jl$, so that the distributions on the two sides of \cref{eq:ramp_message_functional} are conditioned on the same observations -- the functional $\setG$ transforms a belief on $z_{j}$ (given $\Xkj$) into a belief on $z_{j}$ (given $\Xkj$).
Finally, in the case of a singleton observation group, observation-to-latent messages are defined directly by $\rampfzx kp$ as in the RPM.

Thus, the functions $\{\rampfzx kp\}_{(k,p)\in\Ezx}$ and functionals $\{\G\msg.jk\}_{(j,k)\in\Ezz}$ together define a single amortised message-passing network which computes posterior beliefs over each $z_k$ given the observations.

\subsection{Nodewise RPMs}

The key to RAMP is to learn the amortised message-passing factors by simultaneously training multiple nodewise RPMs, one for each $z_{k}$.
These models take the form
\begin{equation}\label{eq:local_rpm_zk}
  \mathsf{P}_k(z_k,\setX) = \rampprior{k}\prod_{j\in\Ne.k}\frac{\rampf{k}{j}\rampemp{k}{j}}{\rampF{k}{j}},
\end{equation}
where $\rampprior{k}$ is the marginal prior on $z_{k}$, $\rampf{k}{j}$ are recognition factors defined by amortised message passing (\cref{eq:ramp_message_functional}), $\rampemp{k}{j}$ is the empirical distribution of $\setX^k_j$, and $\rampF{k}{j}$ are defined by %which result from integrating the recognition factor over the empirical distribution:
\begin{equation}\label{eq:rpm_normaliser_integral}
    \rampF{k}{j} = \intdx[d\setX^k_j]{\rampf{k}{j}\rampemp{k}{j}}
    = \frac{1}{N}\sum_{n=1}^{N}\rampfn{k}{j}{n}\,,
\end{equation}
as in the RPM. The second equality in \cref{eq:rpm_normaliser_integral} follows from the  atomic choice for $\rampempname$.

The nodewise RPM at (each) $z_k$ is a model of the observed data that respects the conditional independences induced by $z_k$. 
Because $z_k$ is latent, fitting the parameters of this model requires optimising a variational free energy (or ELBO).  As the nodewise RPMs have the canonical form introduced by \citet{RPM}, we can directly adopt the form of the free energy derived there:
\begin{equation}\label{eq:nodewise_F}
    \eff_k\nn\lr(){q\nn(z_k),\theta}
    = 
    \angles{\log \rampprior{k} + \sum_{j\in\Ne.k} \log \frac{\rampf kj}{\rampF kj}}_{q\nn}
    \hspace{-1em} + \entropy{q\nn}\,,
\end{equation}
where $\theta$ are the RAMP parameters, $q\nn$ is a variational distribution, and $\entropy{\cdot}$ is the entropy.

% We adopt the ``interior variational'' bound on this free energy (i.e. a lower bound on the lower bound; \citealt{RPM}).  Its form is given in \cref{supp:model}.  For simplicity, we use the symbol $\eff$ below without distinguishing between the two objectives.  

\subsection{Joint optimisation}

Even though the recognition factors of all the nodewise RPMs are defined by a single amortised message-passing network, optimisation of a single $\eff_k$ is not sufficient for accurate learning.  Most obviously, only factors $\G\msg.ij$ that lie along inward paths to $k$ appear in the definitions of $\rampff k{()}$.  More subtly, the $z_k$-nodewise RPM is defined in terms of $\rampemp{k}{j}$, and so does not by itself induce the additional tree-based condition that the different observations within $\Xkj$ be conditionally independent given the latents in $\Tkj$.
Thus, RAMP instead optimises the summed objective
\begin{equation}\label{eq:global_F}
    \eff = \frac{1}{K}\sum_k\eff_k = \frac{1}{KN}\sum_{k,n}\eff^{(n)}_k\,.
\end{equation}
This approach is justified by the following straightforward theorem.

\begin{theorem}\label{thm:functional-ramp}
  Let the distribution $p(\X)$ be compatible with a tree-structured latent model $\T$ satisfying Assumption \ref{assumption:tree_structure}, and let $\Z^* = \{z^*_k\}$ be minimal latent variables in $\T$ such that the beliefs $p(z^*_k | \Xkj)$ are each minimal sufficient statistics of $\Xkj$ for $\bigcup_{l\in\Ne.k\backslash j}\Xkj.kl$.
  Define an amortised message passing network on $\T$ according to \cref{eq:ramp_message_functional}.  
  Then, if $\{\rampprior{k}\}$,  $\{\rampfzx kp\}$ and $\{\G\msg.jk\}$ all optimise the objective of \cref{eq:global_F} in the limit $\rampemp{k}{j} = p(\Xkj)$, we have:
  \begin{enumerate}
    \item each $\eff_k$ attains its optimal value,
    \item the true beliefs $p(z^*_k|\Xkj)$ are each a function of $\rampf kj$.
  \end{enumerate}
\end{theorem}

\begin{proof}
  As $p(\X)$ is compatible with $\T$, the set of possible amortised message-passing networks includes one that computes the true beliefs $p(z^*_k | \Xkj)$.  That is, there exist functions $\{\rampprior{k}\}$,  $\{\rampfzx kp\}$ and $\{\G\msg.jk\}$ defining accurate beliefs in a model that renders the sets $\{\Xkj\}_{j\in\Ne.k}$ conditionally independent given $z_k$ for all $k$.  Thus there is at least one solution that optimises all the $\eff_k$ and (1) follows.  As each nodewise RPM is then an accurate model of the joint over $\{\Xkj\}_{j\in\Ne.k}$, the amortised belief $\rampf kj$ is a sufficient statistic of $\Xkj$ for $\bigcup_{l\in\Ne.k\backslash j}\Xkj.kl$.  Claim (2) then follows as $p(z^*_k | \Xkj)$ is a minimal sufficient statistic by assumption, and so a function of any other sufficient statistic.
\end{proof}

\subsection{Exponential-family parametrisation}

To implement the amortised message passing of \cref{eq:ramp_message_functional} we must represent, multiply, and flexibly transform beliefs $\rampf kj$; and to optimise \cref{eq:nodewise_F}, we must compute integrals over $z_k$.
%,
A natural parametrisation that simplifies both steps is given by a tractable exponential family.
Specifically, we choose $\rampprior{k}$ and $\rampf{k}{j}$ at each $z_k$ to be instances of the same exponential family.
The choice of family may differ for different $z_k$.  Furthermore, the choice of an exponential family for $\rampf kj$ does not restrict the form of joint factor implied by the amortised $\G\msg.jk$.  Instead, the learnt mapping may correspond to an exponential-family projection of an arbitrary joint potential.  
In effect, this choice approximates the marginal posterior over $z_k$ with a tractable parametric form, while allowing the parameters of that posterior to depend on complicated learnt functional mappings from the observations.  Although Theorem \ref{thm:functional-ramp} no longer applies directly, parametric optimisation minimises the sum of nodewise Kullback-Leibler divergences from the optimal model of that result.

Exponential family distributions can be represented by natural parameters. 
In this representation, the product of distributional factors is easy to compute: the result is another distribution in the same family, with natural parameters equal to the sum of the factor natural parameters.

A remaining challenge is to handle the $\rampF kj$, which are mixtures of $\rampf kj$. 
We adopt the ``interior variational'' bound introduced by \citet{RPM}.  
This replaces $F^k_j$ with an exponential family approximation, adding terms to the free energy that ensure it remains a well-behaved lower bound.  Furthermore, based on the observation that $F^k_j \to f^k_0$ as $N\to\infty$ in a well-specified model, we define this bound using $f^k_0$ as the approximating functions.  Further details appear in \cref{supp:model}.

The effect of these assumptions is to make message propagation in RAMP particularly straightforward. If $\eta\msg.ij$ are the natural parameters of $\rampf{j}{i}$ then
\begin{equation}\label{eq:msg_jk_nat_params}
    \eta\msg.jk = g\msg.jk\lr[\Big](){\sum_{i\in\Ne.j\backslash k}\eta\msg.ij},
\end{equation}
where $g\msg.jk$ is a learned function (e.g., a neural network), mapping natural parameters to natural parameters.
Different latent variables may belong to different exponential families:
the only requirement is that all the $g\msg.ij$ functions for fixed $j$ return natural parameters of the family of $z_j$.

\section{Related Work}
There is a long history of methods for approximate inference \citep{sudderth++:cvpr2003,winn+bishop:variational,ihler+mcallester:2009:particle, lienart++:neurips2015} and learning in intractable models. 
Rather than an extensive review, we discuss key approaches that are most relevant to RAMP.

% We briefly review some relevant methods for approximate inference and learning in intractable models.

Black Box Variational Inference \citep[BBVI;][]{bbvi2014} optimises the parameters of a variational distribution by following a stochastic gradient of the free-energy function, employing the \textsc{reinforce} trick to compute Monte-Carlo gradients \citep{williams92reinforce}. 
This method requires the joint model to be specified explicitly, and the variational family is often fully factorised (``naive'' mean-field). 
In contrast, inference in RAMP more directly resembles belief propagation---the recognition factors learn to amortise a marginalisation operator of an implied factor. 
This makes our method better suited to discover structured or hierarchical latent representations (see \cref{subsec:linear_tree,subsec:nonlinear_tree_exp_details}).

Structured variational autoencoders \citep{SVAE,revisitingSVAE} combine nonlinear recognition factors with either tractable exact or mean-field variational message passing on a latent graph.  Similarly, standard applications of the RPM combine nonlinear recognition factors with tractable or variational inference between latents \citep{RPM,RPGSSM}. Again, RAMP combines these nonlinear recognition factors at the leaves with learnt amortised message passing to capture more complex relationships between latents. 

In some ways, the exponential-family approximation of messages in RAMP most closely resembles the framework of expectation propagation \citep[EP,][]{EPMinka2001}.  However, as with variational methods, EP depends on a parametrisation of generative model factors that are then approximated.  For flexibly defined potentials, these approximations may themselves require integrals that are analytically intractable, and numerical approximation introduces error and potential bias \citep{boyen+koller:1999,yu+al:2006:nsspw}. RAMP avoids these challenges by \emph{defining} the model in terms of the projected exponential family belief propagation, exploiting the RPM framework to ensure that the resulting models still have well-defined likelihoods.

Leveraging the ``nodewise'' conditional independences for learning latent representations has been explored in some contrastive methods, particularly in the context of temporal sequences \citep{infonce, eysenbachContrastiveRL}. 
Unlike contrastive methods, RAMP builds an explicit probabilistic model linking observations and latent variables, that can be estimated by (approximate) maximum likelihood. 
The probabilistic framework also makes our method directly applicable to graphical structures that are not necessarily temporal sequences (see \cref{subsec:linear_tree,subsec:nonlinear_tree_exp_details,subsec:pose}), where contrastive methods have been less explored, and are likely to require ad-hoc adjustments.

The idea of composing neural networks in a graph-defined structure has been explored extensively.
In the context of sequences, intuition about the forward-backward structure of inference has inspired Bidirectional RNN models \citep{biRNN,proteinBiRNN,biLSTM,biRNNsLM2015}. 
Applied to time series, RAMP can also be seen as a bidirectional RNN, where the hidden state has direct probabilistic and inferential semantics (see \cref{subsec:pendulum}). Beyond sequences, Graph Neural Networks \citep[for review, see][]{graphNNs_review} have been proposed for representation learning from richly structured data. However, these typically rely on a synchronised message-passing computation, limiting their applicability for inference problems \citep{gnns_hogwild2024}. Rather than iterating a global ``graph layer'', in RAMP each edge has its own neural network, and the message-passing flow is asynchronous, as in classic belief-propagation.

% \lior{Still missing:
% \begin{enumerate}
%     \item The SAVE paper(s) \citep{SVAE, revisitingSVAE}, the Deep Kalman Filters \citep{deepKalmanFilters}, and RPGSSM \cite{RPGSSM}
%     \item EP \citep{EPMinka2001}
% \end{enumerate}
% }

\section{Results}\label{sec:results}
\subsection{Hierarchical Models}\label{subsec:linear_tree}

We begin with a simple motivating example of hierarchical latents models, with
the graphical structure shown in
\cref{subfig:tree_structure}. We use this family of
problems to exemplify the key components of our method, as well as its
robustness to structural and distributional misspecification.

\paragraph{Linear Gaussian Trees}
Jointly Gaussian data were generated from a tree by sampling
a root latent from $\mathcal{N}(0,1)$, and children from $z|\text{pa}(z) \sim \mathcal{N}(a\cdot\text{pa}(z), \sigma^2)$. 
Leaf nodes were observed.
Joint Gaussianity means that exact inference is tractable (given the true parameters of the model).
Nonetheless, learning is still non-trivial due to the existence of latent
variables; moreover, the natural parameter transformations of message passing
are nonlinear, providing a nontrivial test of RAMP. Trained on this data, RAMP
recovered the exact solution with high fidelity.
(\cref{subfig:tree_linear_corrs}).

\paragraph{Structural Misspecification}
We next consider the scenario in which the graphical structure assumed in the
model does not perfectly match that of the true data generating process.
The main learning signals in RAMP are the nodewise free energies, which encourage
learned latents to capture information that is shared between the variables in
different associated subtrees (thereby rendering those subtrees conditionally
independent; \cref{thm:functional-ramp}). That
optimisation pressure should remain accurate at any node that implies a
partition on observations that is correct, and useful when violations relative to
the true structure are few, even if each subtree (internal) structure is misspecified.

To test this claim, we fit RAMP models based on randomly perturbed trees to 
data generated from a balanced base tree (see
\cref{subsec:structural_misspecification_details}).
As the labeling of latents is arbitrary, for analysis each latent in a
perturbed tree was matched to the best candidate in the original tree, by
minimising the partition mismatch metric. This metric counts how many
observations are routed to the wrong sub-tree. In general, the ability of RAMP
to recover the latent variable degraded as the partition mismatch increased.
However, latents which were structurally intact (partition mismatch 0)
\textit{within} a perturbed tree maintained an almost identical correlation
with the corresponding ground-truth latent as that obtained in the
well-specified tree
(\cref{subfig:tree_structural_misspecification}). 

% We generated data as before but used a perturbed tree to define the RAMP
% architectures. Specifically, we iteratively and randomly permuted sub-trees
% from the ground-truth starting point, using Nearest Neighbour Interchange
% (NNI). For each perturbed tree, we measured the global Robinson-Foulds (RF)
% metric to the ground-truth, as well as a per-latent mismatch metric that counts
% how many observations are routed to the wrong sub-tree, from the perspective of
% this latent.
%
% Correlation between inferred posterior means and true latents (averaged over
% all latents) gracefully degrades as the RF distance increases, but does not
% collapse. Moreover,  variables that are
% ``structurally intact'' \textit{within a perturbed tree} (partition mismatch 0)
% maintain an almost identical correlation with corresponding ground-truth latent
% as that obtained by either RAMP (or exact Gaussian inference) in the
% well-specified tree (\cref{subfig:tree_structural_misspecification}).

%

\paragraph{Distributional Misspecification}
Another form of model misspecification is distributional. In RAMP the (marginal) posteriors are constrained to be of an exponential family (here we use Gaussians), which in general would be an approximation. We tested the quality of this approximation by making the data generating process non-linear, thus making the true posteriors non-Gaussian. Here, RAMP outperformed Black Box methods applied to parametrised generative models (\cref{subfig:tree_distributional_misspecification}). 

Additional details and results for this section appear in \cref{subsec:linear_tree_exp_details,subsec:structural_misspecification_details,subsec:nonlinear_tree_exp_details,subsec:linear_tree_additional_results,subsec:nonlinear_tree_additional_results}.

\begin{figure}[t]
    \centering

    \begin{subfigure}{0.15\textwidth}
    \raisebox{0.8cm}[\height]{%
    \includegraphics[width=\columnwidth,page=4]{figures/tikz_figs/tree_conditional_independence.pdf}%
  }%
        \caption{}
        \label{subfig:tree_structure}
    \end{subfigure}%
    \begin{subfigure}{0.25\textwidth}
    \includegraphics[width=\columnwidth]{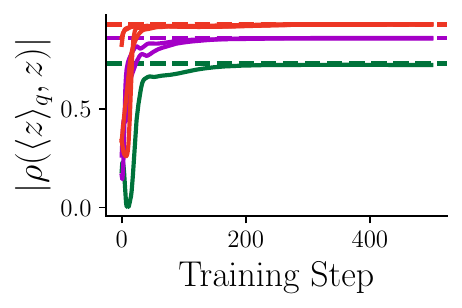}
        \caption{}
        \label{subfig:tree_linear_corrs}
    \end{subfigure}%
    \begin{subfigure}{0.25\textwidth}
      \includegraphics[width=\columnwidth]{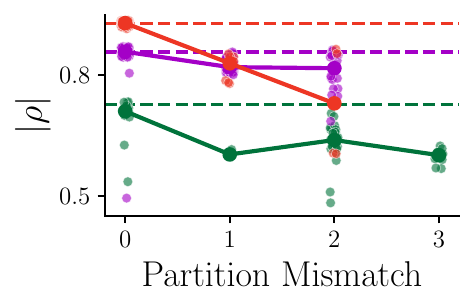}%
      \caption{}
      \label{subfig:tree_structural_misspecification}
    \end{subfigure}%
    \begin{subfigure}{0.35\textwidth}
      \includegraphics[width=\columnwidth]{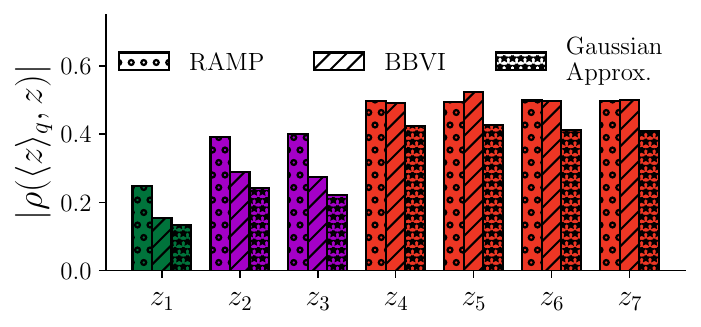}%
      \caption{}
      \label{subfig:tree_distributional_misspecification}
    \end{subfigure}

    \caption{
      (\subref{subfig:tree_structure}) Hierarchical structure with 3 levels of latent variables (color coded) arranged on a binary tree and observations on the leaves. 
      (\subref{subfig:tree_linear_corrs}) Under a linear-Gaussian data generating process, Pearson correlations between the posterior means inferred by RAMP and the ground truth latents (solid lines) reach the optimal bound achieved by exact inference with the true generative model (dashed lines).
      (\subref{subfig:tree_structural_misspecification}) Data as in (\subref{subfig:tree_linear_corrs}), but RAMP is trained using misspecified trees. The quality of representations learned at individual latents degrade with the amount by which they are perturbed (see main text). Crucially, latents that are structurally intact (partition mismatch 0) maintain their unperturbed performance (dashed lines), even though their sub-trees are misspecified. Solid lines trace the median level of correlations within each group of latents.
      (\subref{subfig:tree_distributional_misspecification}) Non-linear data generating process. RAMP learns more accurate representations compared to a black box method and a (supervised) Gaussian approximation, measured by the Spearman correlation of inferred posterior means and true latents.
    }
    \label{fig:hierarchical_models}
\end{figure}

\subsection{Nonlinear Dynamical Systems}\label{subsec:pendulum}
An important application of latent variable models is to time-series data, under the general assumption of a Markovian state-space model; 
the sequence of temporally correlated observations $x_{1:T}$ is assumed to be generated by a hidden Markov process, $z_{1:T}$, and an instantaneous emission process.
The graph thus comprises a chain of latent variables, each connected to a single (unique) observation.
Conditioning on the latent variable at $z_t$ should make $\lbrace x_{1:t-1}\rbrace$, $\lbrace x_t \rbrace$, and $\lbrace x_{t+1:T}\rbrace$ independent.
In other words, knowing the hidden state of the system at time $t$ decouples past, present, and future observations.
The corresponding message-passing process has a forward-backward structure: each node $z_t$ receives a message from the observation directly connected to it, $x_t$, and from $z_{t-1}$ and $z_{t+1}$.\footnote{The node at the first (last) time-step does not receive an incoming forward (backward) message.}
Time-series models often assume time invariance: i.e., that the transition probability $p(z_{t+1}|z_t)$, and the emission probability $p(x_t|z_t)$, do not depend on $t$.  In RAMP this implies that the parameters of the neural networks propagating ``forward'', ``backward'', and ``observation'' messages are tied:
\begin{equation*}
    g\msg.t{t+1} \equiv g_\text{fwd}, \quad
    g\bck.t{t+1} \equiv g_\text{bwd}, \quad
    g\msg.{x_t}t  \equiv g_\text{obs}.
\end{equation*}
With this, the amortised message-passing computation becomes a bidirectional RNN, processing the sequence of observation potentials (i.e., $g_\text{obs}(x_t)$). 
However, instead of arbitrary hidden-state activations, these RNNs communicate probabilistic beliefs over $z_t$ represented by natural parameters. %, which will often be of considerable lower dimensionality.

\begin{figure*}[t]
    \centering
    \includegraphics[width=0.8\textwidth]{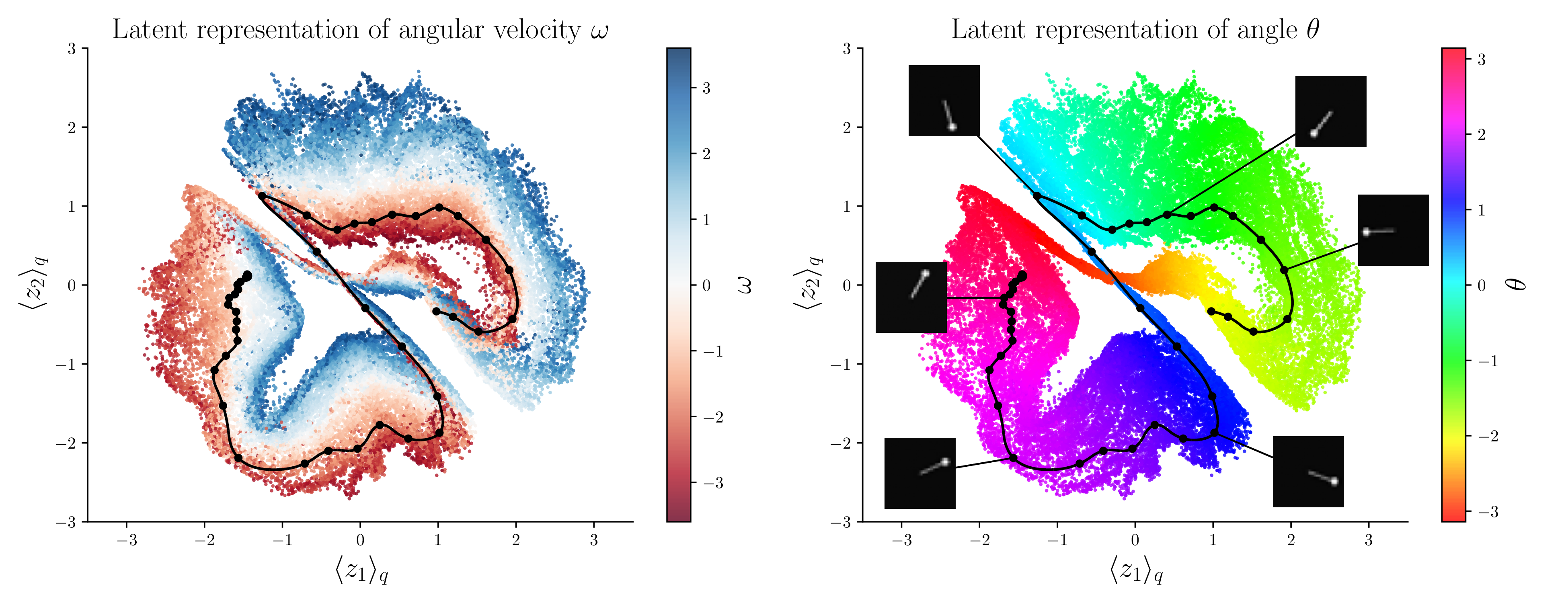}
    \caption{
    Inferred posterior means of learnt 2D latents for the pendulum dataset. Left: points coloured by ground-truth angular velocity. Right: points coloured by ground-truth angle. The black line tracks a single trajectory, with representative observations displayed. RAMP learns a latent representation of both dynamical state dimensions, reliably separating angles and angular velocities.
    % \lior{coloured scatters: all dataset frames, coloured by their corresponding true latent ($\omega$ / $\theta$), organised by the 2D posterior mean inferred by RAMP at this frame. Black dots: example of one consecutive (sub-) trajectory, black line: interplotation between the dots, for visualization only; images: corresponding observations}
    }
    \label{fig:pendulum}
\end{figure*}

We trained a RAMP model on  image observations generated by a nonlinear dynamical system.
We chose a simple physical pendulum with the dynamics $\dot{\theta}=\omega$; $\dot{\omega}=-\sin(\theta)$,
% \begin{equation*}
%     \dot{\theta}=\omega,\quad\quad\dot{\omega}=-\sin(\theta),
% \end{equation*}
where $\theta$ is the angle of the pendulum and $\omega$ its angular velocity. 
The observation $x(t)$ at each time was taken to be a rendered image of a pendulum at angle $\theta(t)$. 
\begin{wraptable}{r}{0.4\textwidth}
  \centering
  \caption{$R^2$ values (top: train; bottom: test) Kernel Ridge Regression from inferred latents to ground truth latents.}
  \label{tab:pendulum-baselines}
 \renewcommand{\arraystretch}{1.2}
 % \smaller
  \begin{tabular}{lccc}
    \toprule
     & $\cos \theta$ & $\sin \theta$ & $\omega$ \\
    \midrule
    RAMP & \makecell{0.9868\\0.9857} & \makecell{0.9434\\0.9459} & \makecell{0.7127\\0.6444} \\
    \midrule
    SVAE & \makecell{0.9983\\0.9908} & \makecell{0.9976\\0.9882} & \makecell{0.2021\\-0.1856} \\
    \midrule
    CPC  & \makecell{0.8131\\0.8047} & \makecell{0.1208\\0.1001} & \makecell{0.0396\\0.0097} \\
    \bottomrule
  \end{tabular}
\end{wraptable}
%$
Two sources of nonlinearity in this problem pose a challenge:
the mapping between angle and image, and the dynamical law of the hidden state. 
The first requires recognition factors $g_\text{obs}$ expressive enough to extract angle information from the image.
But angle alone is not enough to make past, future and present conditionally independent.  Evolution depends on the angular velocity, which cannot be read from a static frame.  Thus the model must learn to infer it by temporal integration.
This makes the ability to learn nonlinear transformations between latent variables essential. As the underlying system is nonlinear, linear state-space methods typically rely on a higher-dimensional latent embedding\citep{koopman:brunton_natcomm2018,koopman:review2021,koopman:KAE_nayak2025,RPGSSM,SVAE}. 
When limited to the true two-dimensional latent space, both Contrastive Predictive Coding \citep[CPC;][]{infonce} and SVAE failed to reliably recover angular velocity information.
In contrast, RAMP recovered the phase space of the system within a two-dimensional latent space (\cref{tab:pendulum-baselines}). This embedded phase-space can be directly visualised: \cref{fig:pendulum} shows posterior means inferred by a trained model for all frames in the dataset, coloured by the corresponding true angle and/or angular velocity.\footnote{The use of ground truth hidden state is for visualization alone; the model is trained only from image observations.}

\begin{figure}
    \centering

    \begin{tikzpicture}[every node/.style={inner sep=0pt}]

    \node (A) [left=1em] {%
    \raisebox{4ex}{\includegraphics[width=0.2\linewidth,trim=0 200 0 0,clip]{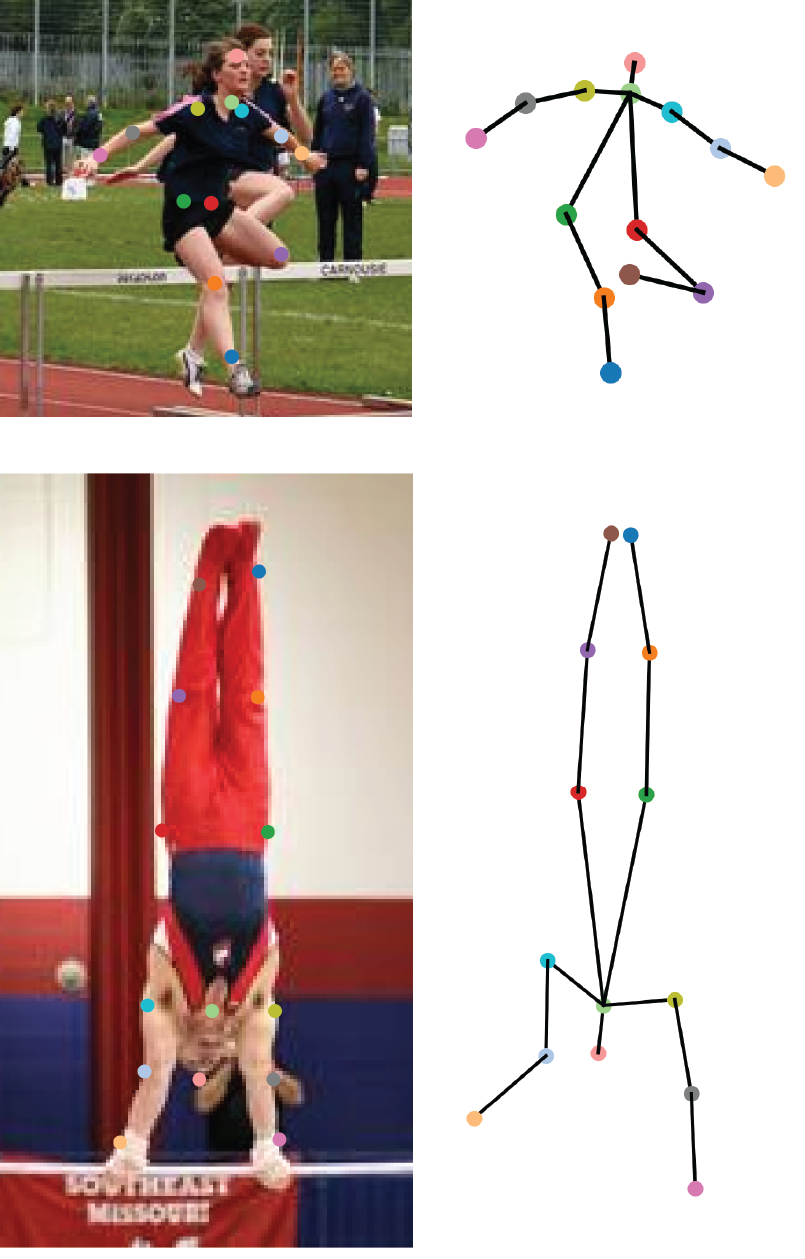}}
    \includegraphics[width=0.2\linewidth,trim=0 0 0 100,clip]{figures/Pose.png}%
    };
    \path[fit to node = A] (0,1) node[below right] (A label) {(a)};
    
    \node (B) [right=1em] {\includegraphics[width=0.4\linewidth]{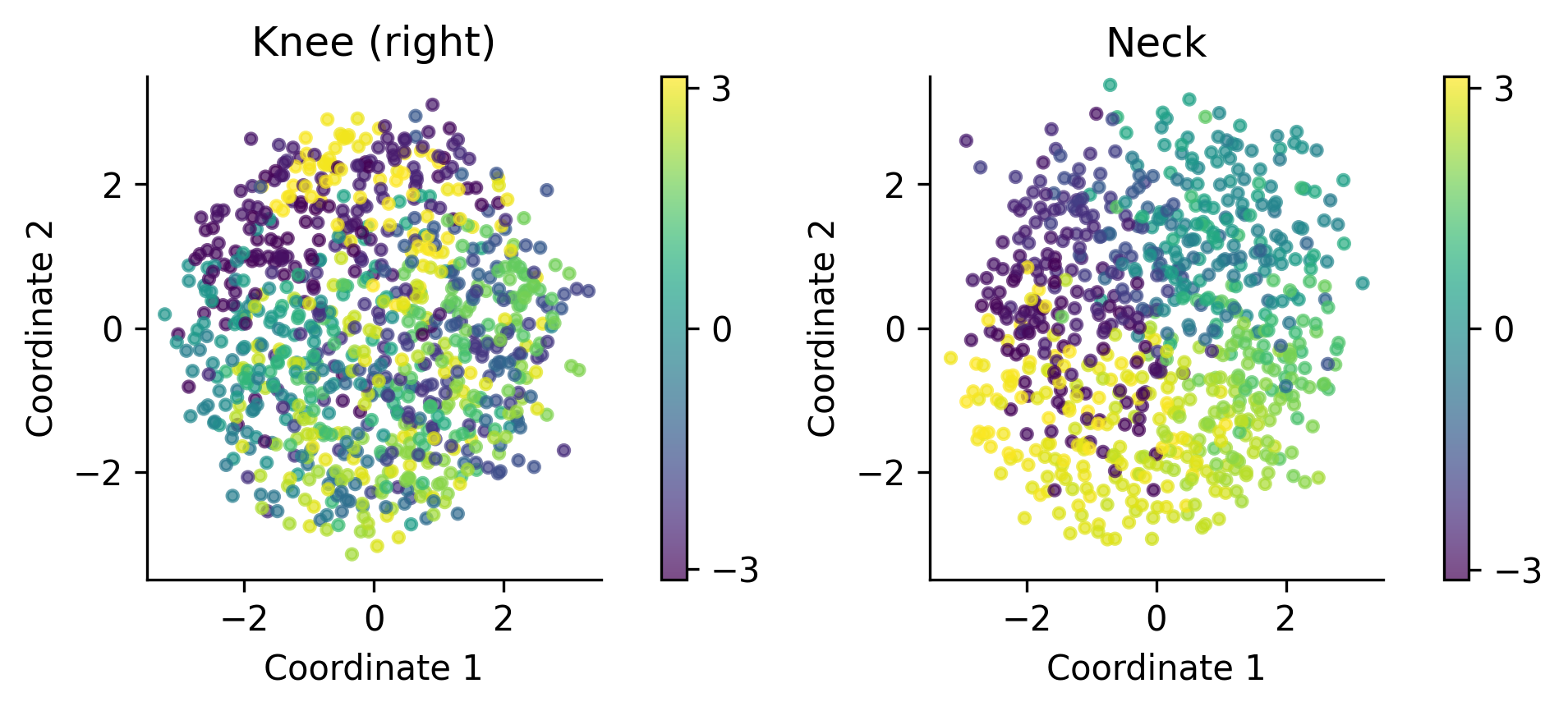}};
    \path[fit to node = B] (0,1 |- A label) node[right] {(b)};
    \end{tikzpicture}

    \caption{
    Pose reconstruction. 
    (a) Images from LSP (left) were used to create simplified stick-figure representations, and local patches extracted based on labelled joint positions (coloured circles). By regressing the inferred posterior means at each joint to edge orientations, a sketch of the underlying pose was reconstructed (right).
    (b) 2D projections of the 6D posterior means inferred at the right knee (left) and neck (right), coloured by the angle (in radians) of the edge connecting the joint to the right ankle (left) and right hip (right), respectively. Each dot represents a different image.
    }
    \label{fig:pose_recon}
    \label{fig:MDS}
\end{figure}

\subsection{Pose Estimation}\label{subsec:pose}

We next considered the problem of human pose estimation,
where dependencies are tree-structured in space rather than time.
The goal is to infer body configuration from real-world observations. Since the body forms a tree-like structure 
with joints connected by limbs, 
we used RAMP to model their dependencies and estimate pose.
The graph expresses the kinematic layout of the body. Each node represents a joint, with edges drawn between nodes whose joints are connected by limbs; for example, the left knee node has an edge to the left hip node and the left ankle node (see \cref{fig:pose_recon}a).
Each node is associated with its own local observation: a small 32 $\times$ 32 image patch centred on the associated joint. Thus, RAMP does not observe the entire image nor the joint locations, just their local appearances. The model must estimate pose by piecing together image fragments.

We extracted naturalistic poses from the Leeds Sports Pose (LSP) benchmark \citep{Johnson11}.
%, which contains images of sports people performing a wide range of activities.
Due to variation in body pose and viewpoint, joints are often occluded, making knowledge of body structure important for effective inference.
For simplicity, each image was converted to a 
stick-figure representation, with black lines drawn between joints on a white background.
Analysis of the latent variables inferred by RAMP revealed a strong correlation between the inferred posterior means at each joint and the orientations of the edges connecting that joint to its neighbours (\cref{supp:exp-results}). % Table \ref{tab:joint_r2}
This information, which is critical for pose estimation, ensures that the observation at each joint is independent of the observations at neighbouring joints; the information that is shared by a joint and its neighbours is the orientation of the edges that connect them. Edge orientation encoding was also observed in low-dimensional projections of the inferred posterior means at each joint, obtained by multidimensional scaling (Figure \ref{fig:MDS}b; additional examples in \cref{supp:exp-results}).
Based on the orientations encoded at each joint, we reconstructed a sketch of the underlying pose (Figure \ref{fig:pose_recon}a).  Despite the left ankle being occluded in the left image (modelled as a missing observation), RAMP uses the learned tree-structured prior to accurately infer the pose of the left lower leg.

\section{Discussion}

Statistical regularity is key to learning latent representations.
RAMP exploits this signal to learn a complex, non-linear latent tree.
It does so by constructing latents that each model a different correlation pattern in the data, but through beliefs that are coupled by a coherent inference process.

The tree structure of the underlying  graphical model, as assumed in this work, is essential both to ensure identifiability, and to guarantee the coherence of inference.
Nevertheless, belief propagation on \textit{loopy} graphs has been demonstrated to be effective empirically, and some convergence properties are known \citep{pearl:1988,murphy+weiss+jordan:loopyBP, yedidia+freeman+weiss:generalizedBP, wainwright:tree-reparam-loopy-graphs}.
Future work will explore the possibility of adapting RAMP to learn loopy graphical models.

\section{Acknowledgments}
This work was supported by the Gatsby Charitable Foundation (GAT3850, GAT4058), the Simons Foundation (SCGB543039) and the UK Engineering and Physical Sciences Research Council (EPSRC) grant EP/S021566/1 for the UCL Centre for Doctoral Training in Foundational Artificial Intelligence.

\bibliography{journals,sources}

\begin{thebibliography}{37}
\providecommand{\natexlab}[1]{#1}
\providecommand{\url}[1]{\texttt{#1}}
\expandafter\ifx\csname urlstyle\endcsname\relax
  \providecommand{\doi}[1]{doi: #1}\else
  \providecommand{\doi}{doi: \begingroup \urlstyle{rm}\Url}\fi

\bibitem[Arisoy et~al.(2015)Arisoy, Sethy, Ramabhadran, and Chen]{biRNNsLM2015}
Ebru Arisoy, Abhinav Sethy, Bhuvana Ramabhadran, and Stanley Chen.
\newblock Bidirectional recurrent neural network language models for automatic speech recognition.
\newblock In \emph{2015 IEEE International Conference on Acoustics, Speech and Signal Processing (ICASSP)}, pages 5421--5425. IEEE, 2015.

\bibitem[Baldi et~al.(1999)Baldi, Brunak, Frasconi, Soda, and Pollastri]{proteinBiRNN}
Pierre Baldi, Søren Brunak, Paolo Frasconi, Giovanni Soda, and Gianluca Pollastri.
\newblock Exploiting the past and the future in protein secondary structure prediction.
\newblock \emph{Bioinformatics}, 15\penalty0 (11):\penalty0 937--946, 11 1999.
\newblock ISSN 1367-4803.
\newblock \doi{10.1093/bioinformatics/15.11.937}.
\newblock URL \url{https://doi.org/10.1093/bioinformatics/15.11.937}.

\bibitem[Bogo et~al.(2016)Bogo, Kanazawa, Lassner, Gehler, Romero, and Black]{BogoSMPL}
Federica Bogo, Angjoo Kanazawa, Christoph Lassner, Peter Gehler, Javier Romero, and Michael~J. Black.
\newblock Keep it {SMPL}: Automatic estimation of 3{D} human pose and shape from a single image.
\newblock In Bastian Leibe, Jiri Matas, Nicu Sebe, and Max Welling, editors, \emph{Computer Vision -- ECCV 2016}, pages 561--578, Cham, 2016. Springer International Publishing.
\newblock ISBN 978-3-319-46454-1.

\bibitem[Boyen and Koller(1999)]{boyen+koller:1999}
Xavier Boyen and Daphne Koller.
\newblock Approximate learning of dynamic models.
\newblock In M.~S. Kearns, S.~A. Solla, and D.~A. Cohn, editors, \emph{Advances in Neural Information Processing Systems}, volume~11. MIT Press, 1999.

\bibitem[Brunton et~al.(2021)Brunton, Budi{\v{s}}i{\'c}, Kaiser, and Kutz]{koopman:review2021}
Steven~L Brunton, Marko Budi{\v{s}}i{\'c}, Eurika Kaiser, and J~Nathan Kutz.
\newblock Modern koopman theory for dynamical systems.
\newblock \emph{arXiv preprint arXiv:2102.12086}, 2021.

\bibitem[Dempster et~al.(1977)Dempster, Laird, and Rubin]{dempster+al:1977}
A.~P. Dempster, N.~M. Laird, and D.~B. Rubin.
\newblock Maximum {Likelihood} from {Incomplete} {Data} via the {EM} {Algorithm}.
\newblock \emph{Journal of the Royal Statistical Society. Series B (Methodological)}, 39\penalty0 (1):\penalty0 1--38, 1977.
\newblock ISSN 0035-9246.

\bibitem[Eysenbach et~al.(2022)Eysenbach, Zhang, Levine, and Salakhutdinov]{eysenbachContrastiveRL}
Benjamin Eysenbach, Tianjun Zhang, Sergey Levine, and Russ~R Salakhutdinov.
\newblock Contrastive learning as goal-conditioned reinforcement learning.
\newblock In S.~Koyejo, S.~Mohamed, A.~Agarwal, D.~Belgrave, K.~Cho, and A.~Oh, editors, \emph{Advances in Neural Information Processing Systems}, volume~35, pages 35603--35620. Curran Associates, Inc., 2022.
\newblock URL \url{https://proceedings.neurips.cc/paper_files/paper/2022/file/e7663e974c4ee7a2b475a4775201ce1f-Paper-Conference.pdf}.

\bibitem[Firth(1957)]{firth1957}
John Firth.
\newblock A synopsis of linguistic theory, 1930-1955.
\newblock \emph{Studies in linguistic analysis}, pages 10--32, 1957.

\bibitem[Graves and Schmidhuber(2005)]{biLSTM}
Alex Graves and Jürgen Schmidhuber.
\newblock Framewise phoneme classification with bidirectional lstm and other neural network architectures.
\newblock \emph{Neural Networks}, 18\penalty0 (5):\penalty0 602--610, 2005.
\newblock ISSN 0893-6080.
\newblock \doi{https://doi.org/10.1016/j.neunet.2005.06.042}.
\newblock URL \url{https://www.sciencedirect.com/science/article/pii/S0893608005001206}.
\newblock IJCNN 2005.

\bibitem[Hromadka et~al.(2025)Hromadka, Biegun, Fox, Heald, and Sahani]{RPGSSM}
Samo Hromadka, Kai Biegun, Lior Fox, James Heald, and Maneesh Sahani.
\newblock Maximum likelihood learning of latent dynamics without reconstruction, 2025.
\newblock URL \url{https://arxiv.org/abs/2505.23569}.

\bibitem[Ihler and McAllester(2009)]{ihler+mcallester:2009:particle}
Alexander Ihler and David McAllester.
\newblock Particle belief propagation.
\newblock In \emph{Proceedings of the Twelfth International Conference on Artificial Intelligence and Statistics (AISTATS)}, volume~5, pages 256--263. PMLR, 2009.

\bibitem[Johnson et~al.(2016)Johnson, Duvenaud, Wiltschko, Adams, and Datta]{SVAE}
Matthew~J Johnson, David~K Duvenaud, Alex Wiltschko, Ryan~P Adams, and Sandeep~R Datta.
\newblock Composing graphical models with neural networks for structured representations and fast inference.
\newblock In D.~Lee, M.~Sugiyama, U.~Luxburg, I.~Guyon, and R.~Garnett, editors, \emph{Advances in Neural Information Processing Systems}, volume~29. Curran Associates, Inc., 2016.
\newblock URL \url{https://proceedings.neurips.cc/paper_files/paper/2016/file/7d6044e95a16761171b130dcb476a43e-Paper.pdf}.

\bibitem[Johnson and Everingham(2011)]{Johnson11}
Sam Johnson and Mark Everingham.
\newblock Learning effective human pose estimation from inaccurate annotation.
\newblock In \emph{Proceedings of Computer Vision and Pattern Recognition (CVPR) 2011}, 2011.

\bibitem[Koller and Friedman(2009)]{koller+friedman:2009}
Daphne Koller and Nir Friedman.
\newblock \emph{Probabilistic Graphical Models: Principles and Techniques}.
\newblock MIT Press, Cambridge, MA, 2009.
\newblock ISBN 9780262013192.

\bibitem[Lienart et~al.(2015)Lienart, Teh, and Doucet]{lienart++:neurips2015}
Thibaut Lienart, Yee~Whye Teh, and Arnaud Doucet.
\newblock Expectation particle belief propagation.
\newblock In C.~Cortes, N.~Lawrence, D.~Lee, M.~Sugiyama, and R.~Garnett, editors, \emph{Advances in Neural Information Processing Systems}, volume~28. Curran Associates, Inc., 2015.
\newblock URL \url{https://proceedings.neurips.cc/paper_files/paper/2015/file/a00e5eb0973d24649a4a920fc53d9564-Paper.pdf}.

\bibitem[Lusch et~al.(2018)Lusch, Kutz, and Brunton]{koopman:brunton_natcomm2018}
Bethany Lusch, J~Nathan Kutz, and Steven~L Brunton.
\newblock Deep learning for universal linear embeddings of nonlinear dynamics.
\newblock \emph{Nature communications}, 9\penalty0 (1):\penalty0 4950, 2018.

\bibitem[Minka(2001)]{EPMinka2001}
Thomas~P. Minka.
\newblock Expectation propagation for approximate bayesian inference.
\newblock In \emph{UAI}, pages 362--369, 2001.

\bibitem[Moreno-Noguer(2017)]{Noguer}
Francesc Moreno-Noguer.
\newblock 3d human pose estimation from a single image via distance matrix regression.
\newblock In \emph{2017 IEEE Conference on Computer Vision and Pattern Recognition (CVPR)}, pages 1561--1570, 2017.
\newblock \doi{10.1109/CVPR.2017.170}.

\bibitem[Murphy et~al.(1999)Murphy, Weiss, and Jordan]{murphy+weiss+jordan:loopyBP}
Kevin~P. Murphy, Yair Weiss, and Michael~I. Jordan.
\newblock Loopy belief propagation for approximate inference: an empirical study.
\newblock In \emph{Proceedings of the Fifteenth Conference on Uncertainty in Artificial Intelligence}, UAI'99, page 467–475, San Francisco, CA, USA, 1999. Morgan Kaufmann Publishers Inc.
\newblock ISBN 1558606149.

\bibitem[Nayak et~al.(2025)Nayak, Chakrabarti, Kumar, Teixeira, and Goswami]{koopman:KAE_nayak2025}
Indranil Nayak, Ananda Chakrabarti, Mrinal Kumar, Fernando~L Teixeira, and Debdipta Goswami.
\newblock Temporally-consistent koopman autoencoders for forecasting dynamical systems.
\newblock \emph{Scientific Reports}, 15\penalty0 (1):\penalty0 22127, 2025.

\bibitem[Neal and Hinton(1998)]{neal+hinton:1998}
Radford~M. Neal and Geoffrey~E. Hinton.
\newblock A view of the {EM} algorithm that justifies incremental, sparse, and other variants.
\newblock In Michael~I. Jordan, editor, \emph{Learning in Graphical Models}, pages 355--370. Kluwer Academic Press, 1998.

\bibitem[Pearl(1988)]{pearl:1988}
Judea Pearl.
\newblock \emph{Probabilistic Reasoning in Intelligent Systems: Networks of Plausible Inference}.
\newblock Morgan Kaufmann Publishers Inc., San Francisco, CA, USA, 1988.
\newblock ISBN 1558604790.

\bibitem[Ranganath et~al.(2014)Ranganath, Gerrish, and Blei]{bbvi2014}
Rajesh Ranganath, Sean Gerrish, and David Blei.
\newblock Black box variational inference.
\newblock In \emph{Artificial intelligence and statistics}, pages 814--822. PMLR, 2014.

\bibitem[Sanchez-Lengeling et~al.(2021)Sanchez-Lengeling, Reif, Pearce, and Wiltschko]{graphNNs_review}
Benjamin Sanchez-Lengeling, Emily Reif, Adam Pearce, and Alexander~B Wiltschko.
\newblock A gentle introduction to graph neural networks.
\newblock \emph{Distill}, 6\penalty0 (9):\penalty0 e33, 2021.

\bibitem[Schuster and Paliwal(1997)]{biRNN}
M.~Schuster and K.K. Paliwal.
\newblock Bidirectional recurrent neural networks.
\newblock \emph{IEEE Transactions on Signal Processing}, 45\penalty0 (11):\penalty0 2673--2681, 1997.
\newblock \doi{10.1109/78.650093}.

\bibitem[Solodova et~al.(2024)Solodova, Richardson, Oktay, and Adams]{gnns_hogwild2024}
Olga Solodova, Nick Richardson, Deniz Oktay, and Ryan~P Adams.
\newblock Graph neural networks gone hogwild.
\newblock \emph{arXiv preprint arXiv:2407.00494}, 2024.

\bibitem[Sudderth et~al.(2003)Sudderth, Ihler, Freeman, and Willsky]{sudderth++:cvpr2003}
E.B. Sudderth, A.T. Ihler, W.T. Freeman, and A.S. Willsky.
\newblock Nonparametric belief propagation.
\newblock In \emph{2003 IEEE Computer Society Conference on Computer Vision and Pattern Recognition, 2003. Proceedings.}, volume~1, pages I--I, 2003.
\newblock \doi{10.1109/CVPR.2003.1211409}.

\bibitem[van~den Oord et~al.(2018)van~den Oord, Li, and Vinyals]{infonce}
A{\"a}ron van~den Oord, Yazhe Li, and Oriol Vinyals.
\newblock Representation learning with contrastive predictive coding.
\newblock \emph{ArXiv}, abs/1807.03748, 2018.

\bibitem[Wainwright and Jordan(2008)]{wainwright+jordan:2008}
Martin~J. Wainwright and Michael~I. Jordan.
\newblock Graphical {Models}, {Exponential} {Families}, and {Variational} {Inference}.
\newblock \emph{Foundations and Trends in Machine Learning}, 1\penalty0 (1--2):\penalty0 1--305, 2008.
\newblock ISSN 1935-8237, 1935-8245.

\bibitem[Wainwright et~al.(2001)Wainwright, Jaakkola, and Willsky]{wainwright:tree-reparam-loopy-graphs}
Martin~J Wainwright, Tommi Jaakkola, and Alan Willsky.
\newblock Tree-based reparameterization for approximate inference on loopy graphs.
\newblock In T.~Dietterich, S.~Becker, and Z.~Ghahramani, editors, \emph{Advances in Neural Information Processing Systems}, volume~14. MIT Press, 2001.
\newblock URL \url{https://proceedings.neurips.cc/paper_files/paper/2001/file/9185f3ec501c674c7c788464a36e7fb3-Paper.pdf}.

\bibitem[Walker et~al.(2023)Walker, Soulat, Yu, and Sahani]{RPM}
William~I. Walker, Hugo Soulat, Changmin Yu, and Maneesh Sahani.
\newblock Unsupervised representation learning with recognition-parametrised probabilistic models.
\newblock In Francisco Ruiz, Jennifer Dy, and Jan-Willem van~de Meent, editors, \emph{Proceedings of The 26th International Conference on Artificial Intelligence and Statistics}, volume 206 of \emph{Proceedings of Machine Learning Research}, pages 4209--4230. PMLR, 25--27 Apr 2023.
\newblock URL \url{https://proceedings.mlr.press/v206/walker23a.html}.

\bibitem[Wang et~al.(2014)Wang, Wang, Lin, Yuille, and Gao]{DBLP:conf/cvpr/WangWLYG14}
Chunyu Wang, Yizhou Wang, Zhouchen Lin, Alan~L. Yuille, and Wen Gao.
\newblock Robust estimation of 3{D} human poses from a single image.
\newblock In \emph{2014 {IEEE} Conference on Computer Vision and Pattern Recognition, {CVPR} 2014, Columbus, OH, USA, June 23-28, 2014}, pages 2369--2376. {IEEE} Computer Society, 2014.
\newblock \doi{10.1109/CVPR.2014.303}.
\newblock URL \url{https://doi.org/10.1109/CVPR.2014.303}.

\bibitem[Williams(1992)]{williams92reinforce}
Ronald~J Williams.
\newblock Simple statistical gradient-following algorithms for connectionist reinforcement learning.
\newblock \emph{Machine learning}, 8\penalty0 (3):\penalty0 229--256, 1992.

\bibitem[Winn and Bishop(2005)]{winn+bishop:variational}
John Winn and Christopher~M. Bishop.
\newblock Variational message passing.
\newblock \emph{Journal of Machine Learning Research}, 6\penalty0 (23):\penalty0 661--694, 2005.
\newblock URL \url{http://jmlr.org/papers/v6/winn05a.html}.

\bibitem[Yedidia et~al.(2000)Yedidia, Freeman, and Weiss]{yedidia+freeman+weiss:generalizedBP}
Jonathan~S Yedidia, William Freeman, and Yair Weiss.
\newblock Generalized belief propagation.
\newblock In T.~Leen, T.~Dietterich, and V.~Tresp, editors, \emph{Advances in Neural Information Processing Systems}, volume~13. MIT Press, 2000.
\newblock URL \url{https://proceedings.neurips.cc/paper_files/paper/2000/file/61b1fb3f59e28c67f3925f3c79be81a1-Paper.pdf}.

\bibitem[Yu et~al.(2006)Yu, Shenoy, and Sahani]{yu+al:2006:nsspw}
Byron~M. Yu, Krishna~V. Shenoy, and Maneesh Sahani.
\newblock Expectation propagation for inference in non-linear dynamical models with {P}oisson observations.
\newblock In \emph{Proceedings of the Nonlinear Statistical Signal Processing Workshop}. IEEE, 2006.

\bibitem[Zhao and Linderman(2023)]{revisitingSVAE}
Yixiu Zhao and Scott Linderman.
\newblock Revisiting structured variational autoencoders.
\newblock In Andreas Krause, Emma Brunskill, Kyunghyun Cho, Barbara Engelhardt, Sivan Sabato, and Jonathan Scarlett, editors, \emph{Proceedings of the 40th International Conference on Machine Learning}, volume 202 of \emph{Proceedings of Machine Learning Research}, pages 42046--42057. PMLR, 23--29 Jul 2023.
\newblock URL \url{https://proceedings.mlr.press/v202/zhao23c.html}.

\end{thebibliography}

\clearpage

\appendix
% \title{RAMP: Recognition parametrisation by Amortised Message Passing:\\
%   Supplementary Materials}
% \maketitle

\crefalias{section}{appendix}
% reset counters and add prefix
\setcounter{figure}{0}\renewcommand{\thefigure}{S\arabic{figure}}
\setcounter{table}{0}\renewcommand{\thetable}{S\arabic{table}}
\setcounter{equation}{0}\renewcommand{\theequation}{S\arabic{equation}}

\section{Model Definition and Free Energy}\label{supp:model}

We briefly review the RPM of \citet{RPM} and tree-based message passing, and then
provide additional details about the RAMP model definition, free energy, and
interior variational bound.  Finally, we describe the extension of RAMP to general factor trees.

\subsection{RPM Review}

Suppose we have $N$ observations of a collection of $P$ variables $\X = \{x_1
\dots x_P\}$.  Label each observed value $\X\nn = \{x_1\nn \dots x_P\nn\}$, and
the complete data set $\XN = \{\X\nn[1], \dots \X\nn[N]\}$.  We seek to learn a
model for these data in which the dependence amongst the observations in $\X$ is
captured by one or more latent variables $\Z = \{z_1 \dots z_K\}$.  In other words,
the model should render observations conditionally independent given the
latents, implying a joint distribution of the form
\begin{equation}\label{eq:simple-generative}
  p(\Z,\X) = p(\Z) \prod_p p(x_p | \Z)\,.
\end{equation}
The usual (directed) generative modelling approach is to parametrise each of
the factors in \cref{eq:simple-generative} and fit these parameters by, for example,
maximising the likelihood on the data $\XN$.

Now, by Bayes rule each conditional could instead have been written in the
recognition form $p(x_p | \Z) = p(\Z | x_p) p(x_p) / p(\Z)$. However, this is
not helpful for a fully parametric model: explicitly parametrising $p(x_p)$ may
limit the model marginals to simple distributions; and, moreover, the parameters
of the recognition conditionals $p(\Z | x_p)$ must be constrained to ensure
consistency with the marginals $p(x_p)$ and $p(\Z)$.

The insight of \citet{RPM} was to instead define a \emph{semiparametric} model,
taking $p(x_p)$ to be a nonparametric summary of the corresponding observed
distribution $\rpmemp$, and setting the normalising denominator in the Bayes
rule expression to be the corresponding mixture of recognition factors.
Following those authors' use of the symbol $f$ for the parametrised recognition
factors (which map from observations to parametric---typically exponential 
family---distributions on the latent) and $F$ for the normaliser, the RPM joint is
\begin{align}\label{eq:simple-rpm-appendix}
  p(\Z,\X) &= p(\Z) \prod_p \frac {\rpmf}{\rpmF} \rpmemp &
  &\text{with}&
  \rpmF &= \intdx[dx_p] \rpmf \rpmemp \,.
\end{align}
As $F$ is defined in terms of $f$ and \rpmempname, the only parametrised
distributions in this model are the latent prior $p(\Z)$ and recognition factors
$\{ \rpmf \}_{p=1}^P$; hence the name ``recognition-parametrised model''.

As \citet{RPM} point out, the RPM is well defined for many alternative
nonparametric choices of \rpmempname.  In practice, however, they---and
we---limit implementations to the atomic empirical measure $\rpmemp = \frac 1N
\sum_{n=1}^N \delta(x_p - x_p\nn)$.  In this case $F$ takes the form
\begin{equation}
  \rpmF = \frac 1N \sum_{n=1}^N \rpmfn pn \,.
\end{equation}

The RPM is a properly normalised model with a well-defined likelihood function
on the parameters.  This may be maximised for parameter estimation following the
standard development for latent variable models
\citep{dempster+al:1977,neal+hinton:1998}.  We discuss the variational
free energy and approximations necessary to handle the mixture terms $F$ in
\cref{sec:supp-free-energies}.

\subsection{Tree-Structured Models and Message Passing}

Consider a graphical model with observed variables $\X$, latent
variables $\Z$ (i.e., nodes $\setV = \X\cup\Z$) and edge set $\E = \Ezz \cup \Ezx$.  The general form of the joint is
\begin{equation*}
    p(z_{1:K},x_{1:P}) = \frac{1}{Z}\prod_k\phi_k(z_k)\prod_p\psi_p(x_p)
    \prod_{(k,k')\in\Ezz}\Phi_{kk'}(z_k,z_{k'})
    \prod_{(k,p)\in\Ezx}\Psi_{kp}(z_k,x_p) \,,
\end{equation*}
where $\phi_k$, $\psi_k$, $\Phi_{kk'}$ and $\Psi_{kp}$ are (typically 
parametrised) non-negative factors and $Z$ is a normalising constant.  If the
graph is tree-structured, it is possible to write the same distribution in terms
of its normalised singleton and pairwise marginals,
\begin{equation*}
    p(z_{1:K},x_{1:P}) = \prod_kp(z_k)\prod_pp(x_p)
    \prod_{(k,k')\in\Ezz}\frac{p(z_k,z_{k'})}{p(z_k)p(z_{k'})}
    \prod_{(k,p)\in\Ezx}\frac{p(z_k,x_p)}{p(z_k)p(x_p)} \,,
\end{equation*}
or, defining an arbitrary $z_k$ as the root node of the tree, in terms of directed
conditionals,
\begin{equation*}
  p(z_{1:K},x_{1:P}) = p(z_k)
  \prod_{k'\neq k} p(z_{k'} | \text{pa}(z_{k'}))
    \prod_{p} p(x_p | \text{pa}(x_p)) \,.
\end{equation*}

Belief propagation \citep[BP;][]{pearl:1988} is an efficient message-passing algorithm
to compute various quantities in the graph: these include singleton and pairwise
marginals, posterior marginals and the likelihood.  In particular, defining
potentials at each latent $\alphakj = p(\Xkj|z_j)$ (with $\Xkj$ representing the leaves of the sub-tree rooted at $x_j$ and not including $x_k$, as in \cref{sec:leveraging-ci}), BP provides the recursive updates
\begin{equation}\label{eq:bp-messages-appendix}
\begin{aligned}
  \alphakj &= \intdx[dz_j] p(z_j | z_k) \prod_{l\in\Ne.j\backslash k} p(\Xkj.jl | z_j) 
   = \intdx[dz_j] p(z_j | z_k) \prod_{l\in\Ne.j\backslash k} \alphakj.jl
   &&\text{if } v_j \in \Z
   \\
  \alphakj &= p(x_j | z_k) &&\text{if } v_j \in \X \,.
\end{aligned}
\end{equation}
To compute the marginal posterior $p(z_k | \X)$, we follow  every ``inward'' path from each $x_p$ to $z_k$, computing each message $\alphakj.sr$ for edges $(r,s) \in x_p \leadsto z_k$.  Then
\begin{equation}\label{eq:bp-posterior-appendix}
  p(z_k | \X) = \frac1{L_k} p(z_k,\X) = \frac1{L_k} p(z_k) \prod_{j\in\Ne.k} \alphakj\,,
\end{equation}
where $L_k$ is easily computed by integration over the single variable $z_k$.
Furthermore, $L_k = p(\X)$ and so corresponds to the likelihood of the model
parameters.  Important to the development of RAMP is that this likelihood does not
depend on the choice of node $k$, and so, as long as the messages $\alpha$ can
be computed exactly, the same computation can be performed at any latent node to
yield the same value.  In fact, for reasons discussed below, RAMP optimises all of these likelihoods together.  Fortunately, BP allows efficient computation of every nodewise likelihood: a single ``inward-outward'' pass of messages with respect to any one node in the tree (taking order $|\setV|$ steps) is sufficient to compute them all.

In practice, exact message computation is only possible for a limited class of generative conditionals.  More general (and nonlinear) dependence between variables in the model necessitates approximation \citep{wainwright+jordan:2008}, often combined with Monte-Carlo or numerical integration \citep[e.g.,][]{bbvi2014}.

\subsection{RAMP}

Although the RPM formulation of \citet{RPM} allowed for the model to comprise
multiple latents $\Z$ with their own conditional independence structure, the
recognition parametrisation applied only to the leaf node conditionals from
$x_p$ to (possibly a subset of) $\Z$.  The joint over variables in $\Z$ needed either to be tractable, or to be approximated by standard approaches from
the graphical models literature \citep[cf.][]{wainwright+jordan:2008}.

The key insight of RAMP is that the nodewise BP-derived models of \cref{eq:bp-posterior-appendix} can each be recognition parametrised in a coherent way by directly parametrising the message passing operations of \cref{eq:bp-messages-appendix} without explicit reference to the generative conditionals.
That is, we can parametrise the functional $\setG\msg.jk$ of \cref{eq:ramp_message_functional} (or---for exponential family forms---the natural parameter mapping $g\msg.jk$, as discussed further below),
% of \cref{eq:msg_gmk_natparam} -- no label or eq number as submitted, fix later.
in effect learning an amortised computation of the integrals needed for belief propagation.  
Just as the recognition factors (and prior) \emph{define} the RPM likelihood, these parametrised messages (and priors and leaf recognition factors) \emph{define} (an approximation to) each nodewise RAMP likelihood.

The RAMP model parameters are learnt by maximising a lower bound to the product of the nodewise likelihoods (\crefalt{eq:global_F}), even though in principle each evaluates the same joint probability, as shown above.   
This design is necessitated by the recognition parametrisation of the nodewise RPMs. 

Recognition factors in a single-latent RPM are optimised to infer a belief over the latent variable, capturing information that is shared amongst the different observations.
Thus, a model trained on a single $z_k$-nodewise likelihood would learn amortised messages within each subtree $\setT^k_j$ that carry information shared by variables in $\Xkj$ with variables in a \emph{different set} $\Xkj.kl$, but not necessarily the information shared amongst the variables \emph{within} $\Xkj$. 
This information content is a feature of the learning objective rather than of the topology of message passing.
To ensure the RAMP model captures information linking two variables $x_r,x_s \in \Xkj$ but not shared with any variables in $\Xkj.kl, l\neq j$, it is necessary to also optimise the nodewise RPM likelihood based at a latent node that lies along the direct path $x_r \leadsto x_s$.

In this way, each nodewise RPM likelihood is optimised to create conditional independence within the corresponding partition $\{\Xkj: j\in\Ne.k\}$.  As shown by Lemma~\ref{lemma:local_CIs_suff}, the collection of all these nodewise observation conditional independencies is sufficient to ensure consistency with the entire tree.

% Furthermore, unlike the RPM recognition factors which always define an implicit generative model, the distributional constraints imposed on the amortised message calculation will not generally be consistent with exact belief propagation for any true set of generative conditionals.
% %
% By optimising all the likelihoods together, RAMP favours amortised messages that are as internally consistent as possible.  

\subsection{Free Energy and Interior Bound}\label{sec:supp-free-energies}

The variational free energy (sometimes called the ELBO) is a lower bound on a model likelihood in terms of the parameters $\theta$ that define the model $p(\Z,\X)$, and a variational distribution over the latent variables, $q(\Z)$.  It has the form
\begin{equation}
  \eff(\theta, q) = \angles{\log p(\Z,\X)}_q + \entropy{q} \,,
\end{equation}
where angle brackets indicate expectation and $\entropy{\cdot}$ is the entropy.  The maximum of $\eff$ with respect to $q$ is achieved when $q(\Z) = p(\Z|\X)$ and equals the log-likelihood $\log p(\X)$.

Applied to the $z_k$-nodewise RPM of \cref{eq:local_rpm_zk}, and dropping terms independent of $\theta$ and $q$ we have:
\begin{align*}
    \eff_k\lr(){q(z_k),\theta} 
    &\eqconst 
    \angles{\log \rampprior{k} + \sum_{j\in\Ne.k} \lr[\Big](){ \log \rampf kj - \log \rampF kj}}_{q} + \entropy{q}\\
    &= - \KL{q}{\rampprior{k}} - \sum_{j\in\Ne.k}\KL{q}{\rampf kj} 
      + \sum_{j\in\Ne.k}\KL{q}{\rampF kj}\,,
\end{align*}
where $\KL\cdot\cdot$ is the Kullback-Leibler (KL) divergence.

Direct optimisation of this free energy is made difficult by the $\rampF kj$ factors in two ways.  First, they complicate the form of the optimal variational distribution $q$.  Second, even if $q$ is restricted to a tractable class, as is common in variational learning, exact computation of the final KL term remains challenging.

\DeclareDotOptCommand\tf{2}.{k}{j}{\tilde f^{#1}_{#2}(z_{#1})}
\DeclareDotOptCommand\hf{2}.{k}{j}{\hat f^{#1}_{#2}(z_{#1})}

While \citet{RPM} also investigated sampling and numerical approximation techniques for the KL computation, here we adopt a variant of those authors' ``interior variational bound''.  Introducing variational functions $\tf$ we have
\begin{align}
    \angles{\log \rampF kj}_q 
      &= \angles{\log \rampF kj \frac\tf {q(z_k)}}_q + \angles{\log\frac{q(z_k)}\tf}_q \nonumber
    \\
      &\leq \log\Gamma^k_j + \angles{\log\frac{q(z_k)}\tf}
    \qquad
      \text{with }\Gamma^k_j = \intdx[dz_k] \rampF kj\tf\,.
\end{align}
Thus we obtain a lower bound on the free energy
\begin{align}\label{eq:interior-F}
    \eff_k\lr(){q(z_k),\theta} 
    \ge \widetilde\eff_k\lr(){q(z_k),\theta,\{\tf\}} 
    &=
    \angles{\log \rampprior{k}}_q + \sum_{j\in\Ne.k} \lr(){ \angles[\bigg]{\log \rampf kj + \log \frac{\tf}{q(z_k)}}_{q} - \log \Gamma^k_j} + \entropy{q}\,.
\end{align}

Now, if $\tf$ is chosen to have the form of an unnormalised distribution in the same family as $\rampprior{k}$ and $\rampf kj$, then the distribution $q(z_k)$ which maximises $\tilde \eff_k$ will lie within that same tractable family, while $\Gamma^k_j$ becomes a straightforward sum of exponential family normalisers.  

The bound of \cref{eq:interior-F} will be tight when $\tf \propto q(z_k) / \rampF kj $.  Furthermore, as $\rampF kj$ is an average posterior, we expect its optimal form over learning to approach the prior $\rampprior{k}$.  Thus, rather than optimising $\tilde\eff_k$ with respect to parametric functions $\tf$ we take the approach of setting $\tf = q(z_k) / \rampprior{k} \equiv \hf$ after each computation of the optimal variational posteriors $q^*(z_k)$.  This choice yields the bound
\begin{align*}
    \eff_k\lr(){q(z_k),\theta} 
    &\ge \widetilde\eff_k\lr(){q(z_k),\theta,\{\hf\}} 
    = \angles{\log \rampprior{k}}_q + \sum_{j\in\Ne.k} \lr(){ \angles[\bigg]{\log \frac{\rampf kj}{\rampprior k}}_{q} - \log \widehat\Gamma^k_j} + \entropy{q}
\\
\text{with }
    \widehat\Gamma^k_j &= \intdx[dz_k] \rampF kj \hf \,,
\end{align*}
for which the optimal $q$ has the form
\begin{equation}\label{eq:q_star_zk}
    q^*(z_k) \propto \rampprior{k} \prod_{j\in\Ne.k} \frac{\rampf kj}{\rampprior k}\,.
\end{equation}

% \subsection{RAMP Amortised Messages}

% The choice of interior variational bound for learning also suggests a simplified definition of the amortised messages in RAMP.  
% %
% Recall that the general form of amortised message passing, introduced in the main text, can be written
% \begin{equation*}
%     f^k_j(\cdot|\setX^k_j) = \setG\msg.jk\left(\mathsf P_j(z_j|\setX^k_j) \right)
%     \tag{\ref{eq:ramp_message_functional}}\,.
% \end{equation*}
% Defined in this way, the same concerns would apply to the computation of $\mathsf P_j(z_j|\setX^k_j)$ as to $q$ above.  
% An alternative is to compute the amortised messages using a partial form of $q^*(z_j)$ defined by analogy to \cref{eq:partial-rpm-j}:
% \begin{equation}\label{eq:q_j_neg_k}
%     q_j^{\neg k}(z_j) \propto \rampprior{j} \prod_{l\in\Ne.j\backslash k} \frac{\rampf jl}{\rampprior j}\,.
% \end{equation}
% These local beliefs form an approximation to $\mathsf P_j(z_j|\setX^k_j)$ under the same assumptions as in the interior variational bound for learning.  Nonetheless, their use in amortised message passing for recognition is a distinct choice, albeit one that meshes well with the motivation of RAMP by aligning partial beliefs during  message passing with the nodewise beliefs used to optimise the RPM free energies.

\subsubsection{Exponential-Family Messages}
A practical choice, assumed throughout this paper, is for all the parametrised potentials over latent variable $z_k$ to lie in the same exponential family with sufficient statistic function $T^k$ and density
\[
  f(z_k) = e^{\eta\tr T^k(z_k) - \Phi^k(\eta)}
\]
with respect to a suitable base measure.  Here $\eta$ is the natural parameter and $\Phi^k$ is the log-normalising function.  The family (and therefore functions $T^k$ and $\Phi^k$) may differ for different $z_k$.

Now, multiplying the relevant factors (as in \crefalt{eq:q_star_zk}) reduces to summing their natural parameters $\eta$.
Note that in both \cref{eq:q_star_zk,eq:ramp_message_functional}, the recognition factors always appear divided by the corresponding prior.
In practice, we parametrise this ratio directly, letting $\eta^k_j$ represent the difference in natural parameters between the recognition factor $\rampf{k}{j}$ and the prior $\rampprior{k}$.
With this choice, \cref{eq:q_star_zk} becomes (c.f. \cref{eq:msg_jk_nat_params}):
\begin{align}
    \eta^*_k = \eta^k_0 + \sum_{j\in\Ne.k}\eta^k_j(\setX^k_j) = \eta^k_0 + \sum_{j\in\Ne.k}\eta\msg.jk
\end{align}

Working in natural parameter space also makes it easier to parametrise the functionals $\setG\msg.jk$---these should transform natural parameters of $z_j$ into natural parameters of $z_k$.

In our experiments, the prior terms $\rampprior{j}$ were fixed for each model. Therefore rather than including a constant $\eta^j_0$, we define $\setG\msg.jk$ by a parametric function $g\msg.jk$ such that (as in \cref{eq:msg_jk_nat_params}):
\begin{equation*}
    \eta\msg.jk=g\msg.jk\lr[\Bigg](){\sum_{l\in\Ne.j\backslash k} \eta\msg.lj}
\end{equation*}
The constant offset of $\eta^j_0$ is then implicit in the effective ``bias'' of the neural network defining $g\msg.jk$.

\subsection{Complete Algorithm}

The overall structure of RAMP learning is given by \cref{alg:RAMP}, along with pseudocode for the free energy computation in Listing \ref{alg:RAMP_free_energy}.
Learning follows the gradient of the free energy with respect to the parameters defining the amortised messages.\footnote{And potentially parameters of $\rampprior{k}$; although these were held to fixed values in our experiments without loss of generality.}  Messages are computed by inward-outward propagation along the edges of the tree defining the model.  These messages are sufficient to compute $q_k(z_k)$ at each node, and so evaluate the (interior variational bound to the) free energy.  

Efficient computation of the messages and free energy for exponential family beliefs requires an exponential family object that supports: \textbf{a)} Multiplicative combination (i.e., summing natural parameters); \textbf{b)} Computation of the KL divergence between two distributions from their natural parameters, $\KL{\eta_1}{\eta_2}$; \textbf{c)} Computation of the log-normaliser function, $\Phi(\eta)$; and \textbf{d)} automatic differentiation through all of the above methods.

\begin{algorithm}
    \caption{RAMP}
    \label{alg:RAMP}
    \begin{algorithmic}
        \STATE Initialise the parameters of the message-passing networks $\theta$
        \FOR{$t$ in $T$ training iterations}
            \STATE Sample a minibatch of observations $\mathbb{X}^{(t)}\subseteq\mathbb{X}$
            \FORALL{samples $\setX^{(n)}\in\mathbb{X}^{(t)}$}
                \STATE Compute messages for data sample $\setX^{(n)}$
            \ENDFOR
            \STATE Compute the free energy $\eff$ given the batch of messages (see Listing \ref{alg:RAMP_free_energy})
            \STATE Compute the gradient $\nabla_\theta\eff$ and update the parameters
            
        \ENDFOR
    \end{algorithmic}
\end{algorithm}

\lstset{style=mystyle}
\begin{figure}[t]
\begin{lstlisting}[language=Python, label=alg:RAMP_free_energy,caption={Pseudocode for the free energy calculation}]%%

def node_free_energy(prior, messages):
    # prior:    (unbatched)      natural parameters of f_0
    # messages: batch of [J x N] natural parameters of f_j/f_0
    # q:        batch of [N]     natural parameters of q
    # aux:      batch of [J x N] natural parameters of fhat
    
    f[j]   = prior + messages[j]
    q      = prior + messages.sum(axis=0)
    aux[j] = messages.sum(axis=0)

    gamma[j,n,m]  = Phi(aux[j,n]+f[j,m]) - Phi(aux[j,n]) - Phi(f[j,m])
    logGamma[j,n] = gamma[j,n,n] - logsumexp(gamma[j,n])
    
    F = -kl(q,prior) -kl(q,f) + logGamma
    return F

def free_energy(params, model, X):
    # X: batch of N samples of all P observations

    # compute all the messages (parallelized over batch), 
    # and get the marginal priors at each z_k:
    messages = model.apply(params, X)
    priors   = model.apply(params, method='prior')

    F[k] = node_free_energy(priors[k], messages_into_k)

    return F.mean()
    
\end{lstlisting}
\end{figure}
\FloatBarrier

\subsection{RAMP For General Factor Trees}\label{subsec:ramp_general_factor_graphs}

\cref{sec:model_definition} presented the form of RAMP for tree-structured graphical models with pairwise factors.
Here, we give a brief overview of how RAMP may be generalised to arbitrary tree-structured factor graphs.

A factor graph is an undirected bipartite graph $\setT=(\setV,\setU,\setE)$ where $\setV$ is a set of \textit{variable} nodes, $\setU$ is a set of \textit{factor} nodes, and $\setE$ is a set of edges $\setE\subseteq\setV\times\setU$. 
As before, $\setV$ is partitioned into observation nodes $\setX$ and latent nodes $\setZ$, and we assume that $\setT$ has a tree structure with observations on the leaves. The graphical structure defines a factorisation of the joint distribution (over latents and observations) according to:
\begin{equation*}
    p(\setX,\setZ) = \frac{1}{Z}\prod_{u\in\setU}\phi_u(\Ne.u),
\end{equation*}
where $\phi_u$ are factor potentials, $\Ne.u$ is the set of variable nodes connected to the factor $u$, and $Z$ is a normalisation constant. If $\operatorname{deg}(u)=k$, then the factor potential $\phi_u$ is a $k$-way function. An example is shown in \cref{fig:factor_graph}.

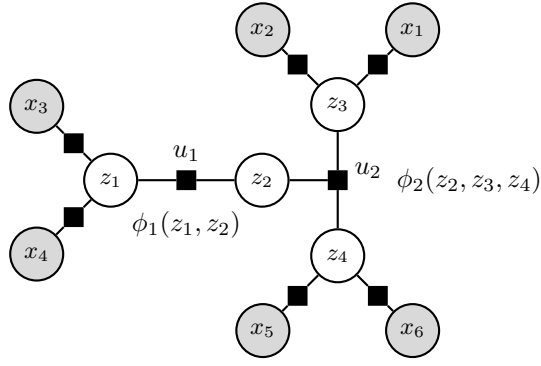
\begin{figure}[h]
    \centering
    \begin{tikzpicture}[
	var/.style={circle,draw,thick,minimum size=7mm,inner sep=0pt,font=\small},
	factor/.style={rectangle,draw,thick,minimum width=2mm,minimum height=2mm,fill=black,text=white,font=\small,align=center},
	fsmall/.style={rectangle,draw,thick,minimum width=2mm,minimum height=2mm,fill=black,text=white,font=\scriptsize},
	obs/.style={circle,draw,thick,fill=gray!30,minimum size=7mm,font=\small},
	edge/.style={thick}
	]
	
	% latent variables
	\node[var] (z2) at (0,0) {$z_2$};
	\node[var] (z3) at (1,1) {$z_3$};
	\node[var] (z4) at (1,-1) {$z_4$};
	
	\node[var] (z1) at (-2,0) {$z_1$};
	
	% pairwise / higher-order factors
	\node[factor,label=$u_1$] (f12) at (-1,0) {};
    \node[below=1ex of f12] {$\phi_1(z_1,z_2)$};
	\node[factor,label={[right=0.1cm]:$u_2$}] (f234) at (1,0) {};
    \node[right=1.5em of f234] {$\phi_2(z_2,z_3,z_4)$};
	
	% connect latent to group factors
	\draw[edge] (z1) -- (f12);
	\draw[edge] (z2) -- (f12);
	
	\draw[edge] (z2) -- (f234);
	\draw[edge] (z3) -- (f234);
	\draw[edge] (z4) -- (f234);
	
	% observations for z1
	\node[fsmall,above left=0.1cm and 0.1cm of z1] (a11) {};
	\node[obs, above left=0.1 cm and 0.1cm of a11]    (y11) {$x_3$};
	\draw[edge] (z1) -- (a11) -- (y11);
	
	\node[fsmall,below left=0.1cm and 0.1cm of z1] (a12) {};
	\node[obs, below left=0.1 cm and 0.1cm of a12]    (y12) {$x_4$};
	\draw[edge] (z1) -- (a12) -- (y12);
	
	% observations for z3
	\node[fsmall,above left=0.2cm of z3] (a31){};
	\node[obs,above left=0.1cm of a31]    (y31) {$x_2$};
	\draw[edge] (z3) -- (a31) -- (y31);
	
	\node[fsmall,above right=0.2cm of z3] (a32) {};
	\node[obs,above right=0.1cm of a32]    (y32) {$x_1$};
	\draw[edge] (z3) -- (a32) -- (y32);
	
	% observations for z4
	\node[fsmall,below left=0.2cm of z4] (a41) {};
	\node[obs,below left=0.1cm of a41]    (y41) {$x_5$};
	\draw[edge] (z4) -- (a41) -- (y41);
	
	\node[fsmall,below right=0.2cm of z4] (a42) {};
	\node[obs,below right=0.1cm of a42]    (y42) {$x_6$};
	\draw[edge] (z4) -- (a42) -- (y42);
	
\end{tikzpicture}
    \caption{A tree-structured factor graph with 7 pairwise factors and one 3-way factor. Two factor nodes are labelled with their corresponding potentials.}
    \label{fig:factor_graph}
\end{figure}

Belief propagation in factor graphs involves passing variable-to-factor and factor-to-variable messages.
The message from a variable node $v$ to one of its neighbouring factors $u\in\Ne.v$ is  given by the belief formed at $v$ by aggregating the messages from all the \textit{other} factors connected to it. Note that this message is a belief \emph{over the variable $v$}:
\begin{equation}\label{eq:m_v_to_u}
    m\msg.vu(v) \propto \prod_{u'\in\Ne.v\backslash u}m\msg.u'v(v).
\end{equation}
The message from a factor node $u$ to one of its neighbouring variables $v$ is the  belief over $v$ derived by marginalising the factor potential with respect to the messages to $u$ from the \emph{other} nodes connected to it:
\begin{equation}\label{eq:m_u_to_v}
    m\msg.uv(v) \propto \intop\lr[\bigg](){\prod_{v'\in\Ne.u\backslash v}\!\!dv'}\; \phi_u(\Ne.u)\prod_{v'\in\Ne.u\backslash v}m\msg.v'u(v').
\end{equation}

The factor-based RAMP model is based on amortisation of the factor-to-variable messages of \cref{eq:m_u_to_v}.
For a $k$-way factor $u$, these could be written generically as a $(k-1)$-way functional that transforms the collection of all incoming messages from variables besides $v$ into a belief over $v$ (c.f. \cref{eq:ramp_message_functional}).  The form of this mapping in a general factor graph is potentially arbitrary; in particular, it is not restricted to a multiplicative combination of the incoming messages. 
%
% Crucially, when parametrising these functionals by neural networks, interaction between the input believs should \textit{not} be restricted to be multiplicative.
%
In practice, for natural-parameter messages, the mapping may be implemented by a neural network acting on the concatenated incoming belief natural parameters to produce the output belief parameters.

To give a concrete example, consider the nodewise RPM at $z_2$ in the graph of \cref{fig:factor_graph}.
This RPM has two recognition factors, one for $\lbrace x_3,x_4\rbrace$, and the other for $\lbrace x_1,x_2,x_5,x_6\rbrace$, corresponding to the sets of all observation nodes downstream from each factor $z_2$ connects to (i.e., $u_1$ and $u_2$).
The recognition factors are then of the form:
\begin{align*}
    f^2_{u_1}(z_2|x_3,x_4) &= \mathcal G\msg.{u_1}{z_2}(\mathsf P_1(z_1|x_3,x_4)) \\
    f^2_{u_2}(z_2|x_1,x_2,x_5,x_6) &= \mathcal G\msg.{u_2}{z_2}(\mathsf P_3(z_3|x_1,x_2), \mathsf P_4(z_4|x_5,x_6) )
\end{align*}

% %
% The first recognition factor is effectively parametrised via an incoming message $g\msg.12$, as we have seen in the case of pairwise factor graphs.
% %
% The second recognition factor 

% The second recognition factor should depend on the beliefs formed at $z_3$ (given $x_2,x_1$) and $z_4$ (given $x_5,x_6$).
% %
% However, these do no form independent recognition factors for $z_2$, because $z_4$ and $z_3$ are \textit{not} conditionally independent given $z_2$.
% %
% Rather, the incoming message into $z_2$ should be given by marginalising the factor $\phi_{234}$ (linking $z_2,z_3$, and $z_4$) over the beliefs at $z_3$ and $z_4$.

% To express this in RAMP, we associate with each $k$-way factor $k$ neural networks that amortise the factor-to-variable messages.
% %
% Each of these networks now take $k-1$ input beliefs and generates a belief over the remaining variable.
% %
% Crucially, interactions between the $k-1$ incoming beliefs should not be restricted to be multiplicative.

\section{Details of Experiments}\label{supp:exp-details}

\subsection{Gaussian Tree}\label{subsec:linear_tree_exp_details}

The hierarchical linear-Gaussian model is defined by the following sequential sampling process:
\begin{align*}
    z_1 &\sim \mathcal{N}(0,1)\\
    v|\text{pa}(v)   &\sim \mathcal{N}(a_v\cdot\text{pa}(v), \sigma_v^2),
\end{align*}
where $z_1$ is the node at the root of the tree, $v$ is any other node in the tree (either a leaf $x_i$ or an internal node $z_i$) and $\text{pa}(v)$ is the parent of $v$.

This model definition holds for general trees.
In our experiments, we used a complete binary tree of depth 4 (including the root). 
We take the leaves to be observations and the rest of the nodes to be latent variables, yielding $8$ observations and $7$ latents, as follows:

\begin{figure}[h!]
    \centering
    \begin{tikzpicture}[level 1/.style={level distance=8mm, sibling distance=40mm},
    					level 2/.style={level distance=8mm, sibling distance=20mm},
    					level 3/.style={level distance=10mm, sibling distance=10mm},
    					edge from parent/.style={->,draw},
                        lat/.style={},
                        obs/.style={fill=gray!30}
    					]
    	\node[circle,draw,lat] {$z_1$}
    		child {node[circle,draw,lat] {$z_2$}
    			child {node[circle,draw,lat] {$z_4$}
    				child {node[circle,draw,obs] {$x_1$}}
    				child {node[circle,draw,obs] {$x_2$}}
    			}
    			child {node[circle,draw,lat] {$z_5$}
    				child {node[circle,draw,obs] {$x_3$}}
    				child {node[circle,draw,obs] {$x_4$}}
    			}
    		}
    		child {node[circle,draw,lat] {$z_3$}
    			child {node[circle,draw,lat] {$z_6$}
    				child {node[circle,draw,obs] {$x_5$}}
    				child {node[circle,draw,obs] {$x_6$}}
    			}	
    			child {node[circle,draw,lat] {$z_7$}
    				child {node[circle,draw,obs] {$x_7$}}
    				child {node[circle,draw,obs] {$x_8$}}
    			}
    		}
    		;
    \end{tikzpicture}
    \caption{The tree structure of for the hierarchical linear-Gaussian model of \cref{subsec:linear_tree}}
    \label{fig:binary_tree}
\end{figure}
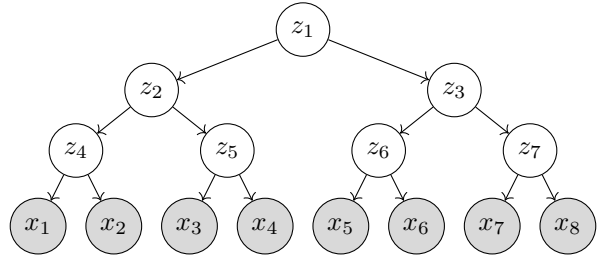
We take the conditional factors to be spatially uniform across the tree, with $a_v\equiv a = 1$ and $\sigma^2_v\equiv\sigma^2=1$.

The linear-Gaussian links imply that the full joint distribution $p(\mathbf{z},\mathbf{x})$ is multivariate Gaussian. 
This multivariate distribution has mean zero and a covariance matrix with a recursive structure, reflecting that of the tree (\cref{fig:binary_tree_covariance_and_precision}).

\begin{figure}[h]
    \centering
    \includegraphics[width=0.35\linewidth]{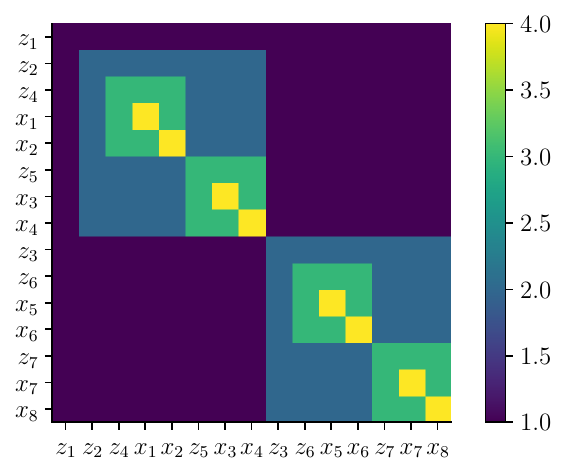}\hspace{2em}%
    \includegraphics[width=0.35\linewidth]{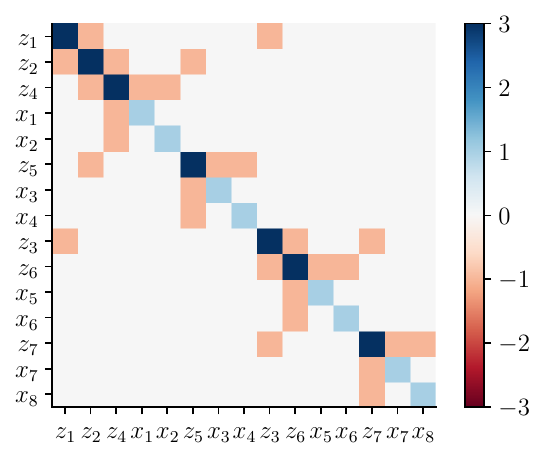}
    \caption{The covariance (left) and precision (right) matrices of the hierarchical linear-Gaussian model of \cref{subsec:linear_tree} (with $a=1$ and $\sigma^2=1$). Non-zero off-diagonal elements in the precision matrix correspond to edges in the graphical model.}
    \label{fig:binary_tree_covariance_and_precision}
\end{figure}

By choosing an appropriate permutation of the variables, the covariance and precision matrices can be written in block form:
\begin{equation*}
    \mathbf{L} = \left[\begin{matrix}
                    \mathbf{L}_\mathbf{zz} & \mathbf{L}_\mathbf{zx} \\
                    \mathbf{L}_\mathbf{xz} & \mathbf{L}_\mathbf{xx}
                \end{matrix}
              \right],
    \quad\quad
    \mathbf{\Lambda} = \mathbf{L}^{-1} = \left[\begin{matrix}
                    \mathbf{\Lambda}_\mathbf{zz} & \mathbf{\Lambda}_\mathbf{zx} \\
                    \mathbf{\Lambda}_\mathbf{xz} & \mathbf{\Lambda}_\mathbf{xx}
                \end{matrix}
              \right],
\end{equation*}
which, using the rules of Gaussian conditioning, give a simple expression for the posterior $p(\mathbf{z}|\mathbf{x})$. 
This posterior is also a multivariate Gaussian, with the statistics (in terms of mean and covariance):
\begin{equation*}
    p(\mathbf{z}|\mathbf{x}) = \mathcal{N}\left(\mathbf{z};\; -\mathbf{\Lambda}_\text{zz}^{-1}\mathbf{\Lambda}_\text{zx}\mathbf{x} , \mathbf{\Lambda}_\text{zz}^{-1}  \right)
\end{equation*}

Naturally, this posterior depends on the prior over $\mathbf{z}$. 
However because these are all latent variables, they can be rescaled arbitrarily, while compensating by the likelihood terms.
From the perspective of RAMP, we have kept the marginal priors ($\rampprior{k}$) of all latent variables fixed at a standard Gaussian with zero mean and unit variance.
Therefore, in order to quantitatively compare the RAMP results with those of the analytic theory, we computed a re-scaled version of the matrix $\mathbf{L}$ (and $\mathbf{\Lambda}$), reflecting the same parametrisation.
For that, define the scaling matrix $\mathbf{S}$ to be the diagonal matrix with:
\begin{equation*}
    S_{ii} = \begin{cases}
        \frac{1}{\sqrt{L_{ii}}} & i\in\mathbf{z} \\
        1                       & i\in\mathbf{x}
    \end{cases},
\end{equation*}
and set
\begin{equation*}
    \widetilde{\mathbf{L}} = \mathbf{SLS^\top}
\end{equation*}
Respectively we define $\widetilde{\mathbf{\Lambda}}=\widetilde{\mathbf{L}}^{-1}$ to get the re-scaled posterior:
\begin{equation}\label{eq:linear_tree_exact_posterior}
    \widetilde{p}(\mathbf{z}|\mathbf{x}) = \mathcal{N}\left(\mathbf{z};\; -\mathbf{\widetilde\Lambda}_\text{zz}^{-1}\mathbf{\widetilde\Lambda}_\text{zx}\mathbf{x} , \mathbf{\widetilde\Lambda}_\text{zz}^{-1}  \right)
\end{equation}

The exact posterior (in either parametrisation) has a non-diagonal covariance matrix, reflecting posterior correlations between different latent variables.
RAMP only gives explicit access to a set of marginal posteriors, with the interactions between different latents being implicitly encoded in the message-passing factors. 
To demonstrate that these interactions are learned by RAMP, we also compared its results to those achieved by the ``mean-field'' variational posterior for the true generative model, i.e. the fully factored distribution $q(\mathbf{z})=\prod_iq(z_i)$ that minimises $\KL{q(\mathbf{z})}{\widetilde{p}(\mathbf{z} | \mathbf{x})}$.
This mean-field solution can be computed analytically for multivariate Gaussians.
It recovers the correct posterior mean, but (generically) underestimates the marginal variances:
\begin{equation}\label{eq:linear_tree_mf_posterior}
    \widetilde{p}_\text{mf}(\mathbf{z}|\mathbf{x}) = \mathcal{N}\left(\mathbf{z};\; -\mathbf{\widetilde\Lambda}_\text{zz}^{-1}\mathbf{\widetilde\Lambda}_\text{zx}\mathbf{x} , 
    \operatorname{diag(\mathbf{\widetilde\Lambda}_{zz})}^{-1}
    \right)
\end{equation}

\Cref{fig:hierarchical_models} shows the correlation between the inferred posterior means and the true latents, for each one of the models: RAMP, variational mean-field, and exact inference.
The mean-field and exact inference models predict constant marginal variances (independent of the observation). While this constraint could be imposed in RAMP, we have kept the parametrisation general, and so for the comparison of marginal variances we have averaged the posterior variances predicted by RAMP, per latent variable, over the dataset.

Hyperparameter settings for the experiment can be found in \cref{tab:experimental_setup_trees}.

\begin{table}[h!]
\label{tab:experimental_setup}
\begin{center}
\begin{footnotesize}
\begin{tabular}{lcccccccc}
\toprule
Dataset & $N$ & \multicolumn{2}{c}{Message-passing networks} & Prior & LR & Batch & Train Iters. & Optim. \\
&& Architecture & Nonlinearity & & & & & \\
\midrule
Linear tree & 10,000 & [32,32] & tanh & $\mathcal{N}(0,1)$ & 0.001 & $N$ & 500 & Adam \\
Nonlinear tree       & 10,000 & [64,64] & tanh & $\mathcal{N}(0,1)$ & 0.001 & $N$ & 500 & Adam\\
\bottomrule
\end{tabular}
\end{footnotesize}
\end{center}
\caption{Hyperparameters for the linear and nonlinear tree experiments (\cref{subsec:linear_tree,subsec:nonlinear_tree_exp_details}}
\label{tab:experimental_setup_trees}
\end{table}

\subsubsection{Black Box Variational Inference}

%Comparing RAMP with the results of exact inference 

In addition to comparing RAMP results to exact inference (which requires knowledge of the true generative model), we also compared to a Black-Box Variational Inference (BBVI) baseline \citep{bbvi2014}.
BBVI enables (variational) inference in analytically intractable models, by following a stochastic, samples-based, estimate of the gradient of the free energy.
The method requires a choice of the variational posterior family that is easy to sample from, as well as the ability to evaluate the log joint-probability over latent and observed variables.

The original method focused on inference in a known generative model.
However, it can naturally be extended to learning model parameters, either by EM, or by jointly optimising variational \emph{and} model parameters by gradient ascent on the free energy.
Here we take the latter approach which more closely resembles learning in RAMP.
We defined the model parametrisation to follow the graphical structure of the true generative model (\cref{fig:binary_tree}), but with conditional factors defined as
\begin{equation*}
    v = f_v(\text{pa}(v)) + \xi_v,
\end{equation*}
where $f_v$ is a neural-network function and $\xi_v$ is independent Gaussian noise with unit variance.
Each $f_v$ was parametrised as a (separate) neural network with one hidden layer of 16 units.
The variational distribution was taken to be a fully-factorised Gaussian, $q(\mathbf{z})=\prod_{i=1}^7q_i(z_i)$ with each $q_i$ a univariate Gaussian.

These choices were made as to resemble the ``structural'' prior of RAMP: that the joint factorises according to the tree structure, that no specific functional/distributional form of the factors themselves is assumed \textit{a priori}, and that the marginal posteriors should be (approximated as) Gaussian distributions.

BBVI was trained with full-batch iterations. We found that BBVI was slower to train compared to RAMP. To reach convergence, we trained for 2,500 training iterations (\cref{fig:bbvi_learning_curve_linear_tree}). For each training iteration and for each example, $100$ samples from the (current) variational posterior were used to form a Monte-Carlo estimate of the free energy.

The same setup was used for the nonlinear tree experiment, with 3,000 training iterations.

%Additionally, we trained a Black-Box Variational Inference (BBVI) algorithm on the Gaussian tree dataset (see Section \ref{subsec:linear_tree_additional_results}). The model comprised a variational distribution as well as a joint generative distribution. The variational component consisted of fully factorised Gaussians, with learnable parameters for the mean and variance based on which we analytically computed the entropy term of the free energy. Additionally, at each iteration, 100 Monte-Carlo samples were extracted to evaluate the expected joint. The joint generative distribution was decomposed according to the tree-defined factorisation, with each term specified using a neural network (architecture [16], non-linearity relu, LR 0.001, batch $N$, Optim Adam) that admitted as input a variational sample and outputted a mean. These means, along with a constant variance of 1, defined the parameters of the conditional Gaussian distributions implied by the generative model. Note that a standard Gaussian was used for the root $z_1$. 

\subsection{Structural Misspecification with Perturbed Trees}\label{subsec:structural_misspecification_details}
To create perturbed trees, we started from the true tree
(\cref{fig:binary_tree}), and applied a random
sequence of $k$ Nearest Neighbour Interchange (NNI) steps. We repeated this
procedure for $k=0,\dots,4$, with $10$ random seeds for each value of $k>0$ (note that $k=0$ is, by
construction, the true unperturbed tree). An illustration of the procedure with $k=2$ is depicted in \cref{fig:nni-procedure}.

Because latents are only defined relative to the tree structure defined on the
observations (consider, for example, a perturbation that would switch $z_{4}$
with $z_{5}$ -- clearly the `perturbed` tree is identical to the original one,
up to relabeling of the latents), in each perturbed tree we have
matched each latent to its best match in the original tree. For a given latent
$z'_{i}$ in the perturbed tree, we have computed for each latent $z_{j}$ in the
true tree the partition mismatch number $m(i,j)$ that counts how many
observations are routed, in the perturbed tree, to the wrong sub-tree at
$z'_{i}$, relative to the partition induced by $z_{j}$ in the true tree. We
then identified $z'_{i}$ with $z_{j}$ that minimised $m(i,j)$. Note that this
matching procedure is not necessarily a permutation -- indeed, multiple latents
in the perturbed tree could be matched to a single latent in the unperturbed
tree. The calculation of partition mismatch index is illustrated in \cref{fig:partition-mismatch}.

\begin{figure}[h]
  \centering
  \includegraphics[width=\linewidth,keepaspectratio]{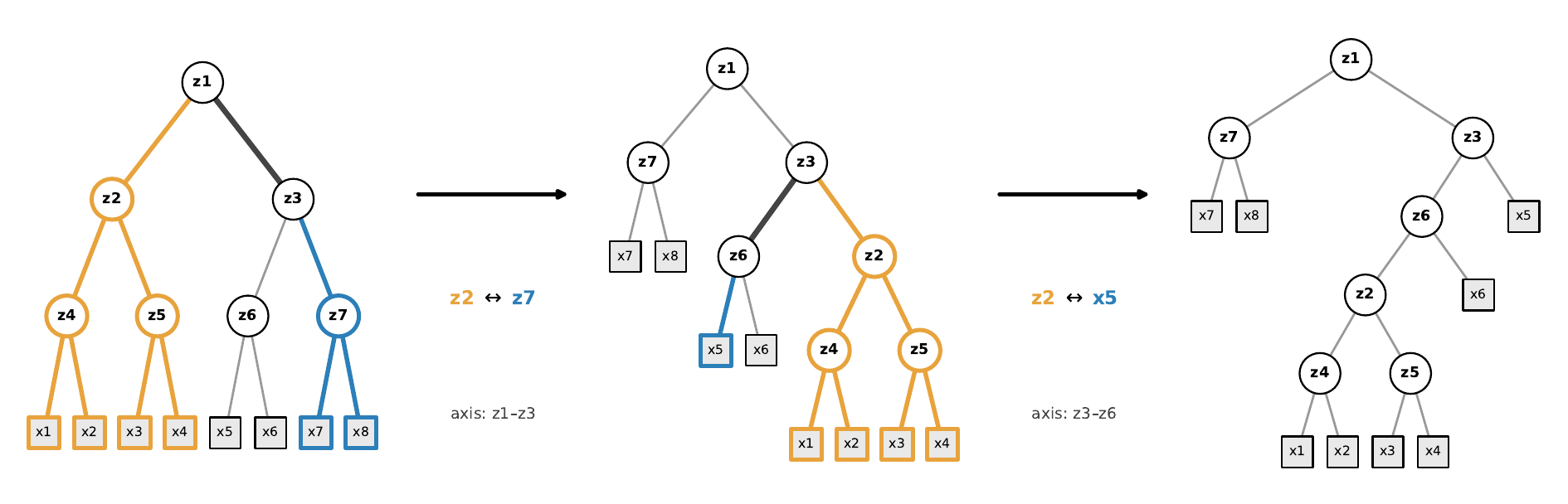}
  \caption{An example of a perturbed tree created by $k=2$ NNI moves}
  \label{fig:nni-procedure}
\end{figure}

\begin{figure}[h]
  \centering
  \includegraphics[width=\linewidth,keepaspectratio]{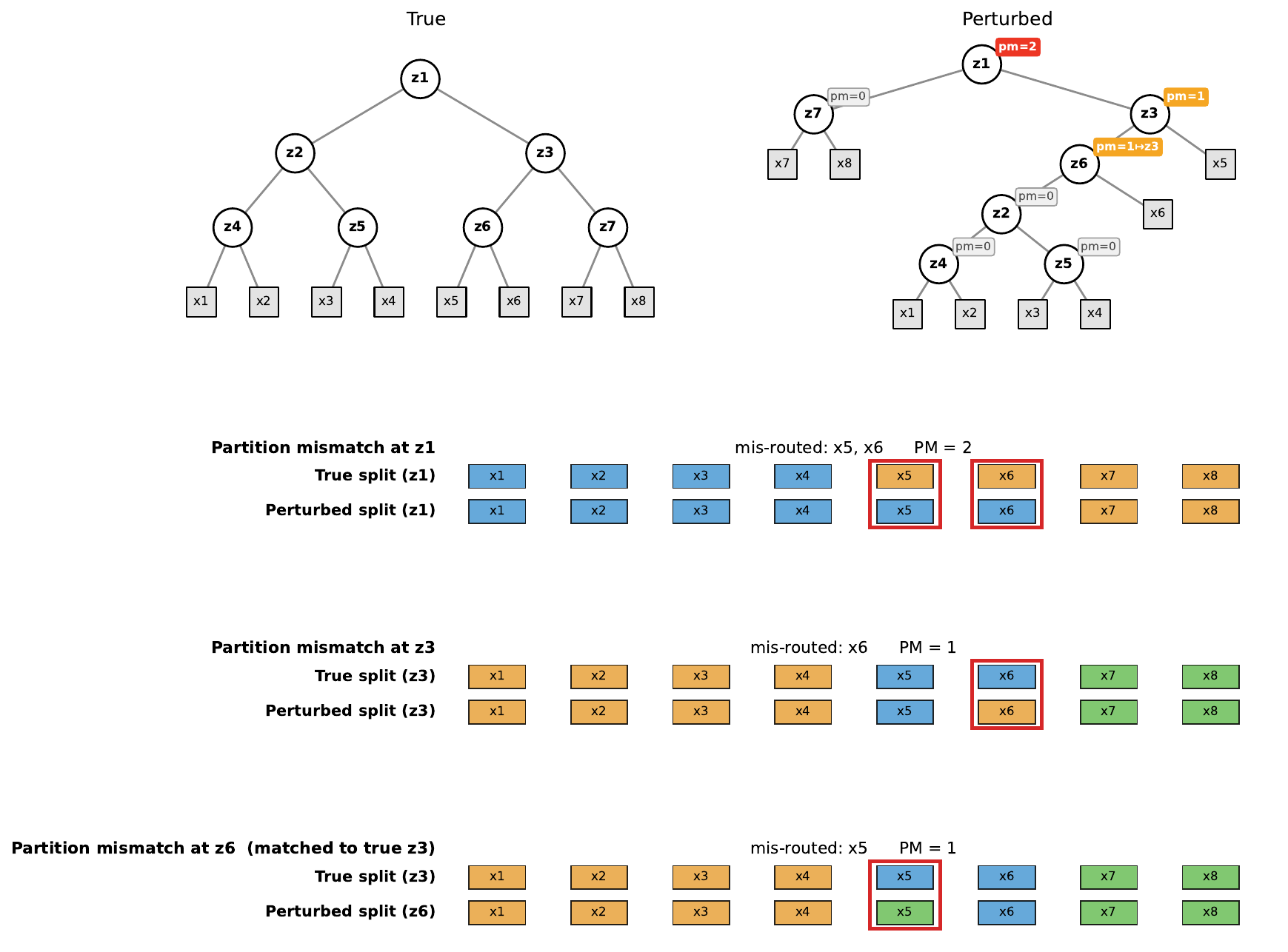}
  \caption{The perturbed tree from \cref{fig:nni-procedure}, with the partition mismatch index of each latent. Note that in the perturbed tree, both $z_{3}$ and $z_{6}$ are matched to the original latent $z_{3}$.}
  \label{fig:partition-mismatch}
\end{figure}

\FloatBarrier
\subsection{Nonlinear Tree}\label{subsec:nonlinear_tree_exp_details}
We next examine a similar example, but with each variable sampled from a \textit{nonlinear} transformation of its parent (a clipped inverse) with additive Gaussian noise.

For the non-linear tree model, we generated data from a hierarchical sampling process with the same tree structure, but with the conditional mean of each node given by a (noisy) non-linear transformation of its parent variable:
\begin{equation*}
    v = f(\text{pa}(v)) + \sigma_v\xi_v.
\end{equation*}

We ensured that the latent variables all have marginal priors with zero mean and unit variance, by ``scaling'' the non-linearities.%
\footnote{note that the marginal prior is \emph{not} a Gaussian: we simply standardized the first two moments}
Specifically, we set:
\begin{equation*}
    v = \sqrt{1-\alpha}\frac{f(\text{pa}(v))-\mu_{\text{pa}(v)}}{\sigma_{\text{pa}(v)}} + \sqrt{\alpha}\xi_v,
\end{equation*}
where $\mu_{\text{pa}(v)},\sigma_{\text{pa}(v)}$ are the mean and standard-deviation of $\text{pa}(v)$ (computed empirically over the sampling batch), $\xi_v$ is independent standard Gaussian noise, and $\alpha$ roughly controls the signal-to-noise ratio in $v$.

In the experiment we set $f(x)=\operatorname{clip}(\frac{1}{x},\left[-15,15\right])$ and $\alpha=0.2$.

For the non-linear model, the joint probability is not multivariate Gaussian, and exact inference is intractable.
We therefore compared RAMP to Black-Box Variational Inference (BBVI), as well as to a Gaussian approximation which constructs an approximate posterior in the same way as in \cref{eq:linear_tree_exact_posterior}, using the 2\textsuperscript{nd} order statistics of the model. 
Note that this Gaussian approximation is a ``supervised'' baseline, since access to the full covariance matrix (including the latent variables) is required.

The hyperparameter settings for the experiment can be found in \cref{tab:experimental_setup_trees}. The BBVI model was the same as that described for the linear tree. Results from this experiment are reported in \cref{subsec:nonlinear_tree_additional_results}.

\subsection{Pendulum}\label{subsec:pend_exp_details}
The state of a simple physical pendulum is characterized by its angle and angular velocity. We sampled trajectories by numerically integrating the dynamical system:
\begin{align*}
    \dot{\theta} &= \omega \\
    \dot{\omega} &= -\sin(\theta)
\end{align*}

We set $T_\text{sim}=10$ and $dt=0.001$. Integration was performed by the \texttt{scipy.integrate.odeint} function with default settings. The initial conditions were sampled uniformly $\theta\in\left[-\pi,\pi\right]$, $\omega\in\left[-3,3\right]$. Note that this generates both oscillating and rotating trajectories of the pendulum. For rotating trajectories, we have interpreted all angles to lie in the interval $[-\pi,\pi]$ for later analysis.

We sampled $N=2,000$ trajectories with random initial conditions, and then sub-sampled each trajectory in intervals of 100 frames, resulting in $T=100$ timesteps in each trajectory. For these 100 timesteps, we have rendered a $128\times128$ grayscale image depicting the pendulum at the corresponding angle (\cref{fig:pendulum_data_viz}).

\begin{figure}[h!]
    \centering
    \includegraphics[width=\linewidth]{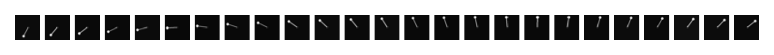}
    \caption{First 25 frames of one of the trajectories in the pendulum dataset}
    \label{fig:pendulum_data_viz}
\end{figure}

In RAMP, we set the marginal prior over each $z_t$ to a 2-dimensional Gaussian distribution (we have also explored different priors; see \cref{subsec:pendulum_additional_results}).
As explained in \cref{subsec:pendulum}, parameters were shared for each group of ``forward'' ($g\msg.t{t+1}$), ``backward'' ($g\bck.t{t+1}$) and ``observations'' ($g\msg.{x_t}t$) amortisers.
Each of those was parametrised by a neural network with 1 hidden layer. We ensured that messages reflected valid Gaussian natural parameters by having the network generate the Cholesky factor of the precision matrix. 

We performed a hyperparameter grid search over the width of these hidden layers and the learning rate (see also \cref{fig:pendulum_sweep}), with the reported run in \cref{subsec:pendulum} having 128 hidden units for all networks and a learning rate of 0.0005. The batch size was kept fixed across runs at 200 sequences per batch.

\subsection{Pose Estimation}

Human pose estimation from a single image typically 
% involves inferring the 3D body pose from the 2D position of the N body joints.
% often 
follows a two-step pipeline. In step one, the 2D position of the $N$ body joints are detected using a computer vision method. In step two, the 2D joint positions detected in step one are used to infer the 3D body pose. 
Step two of the pipeline is typically treated either as a supervised learning problem involving 2$D$-to-3$D$ regression of the Cartesian joint coordinates \citep{Noguer} or as an optimization problem  involving an explicit generative (camera) model of how a 3D pose projects to a 2D image \citep{BogoSMPL,DBLP:conf/cvpr/WangWLYG14}. In this work, we take a fundamentally different approach; pose estimation with RAMP is entirely unsupervised and does not require a generative model to be specified due to recognition parameterisation.

We trained RAMP on a dataset of 10,000 poses derived from the LSP dataset \citep{Johnson11}.
For each original image in the dataset, we created a simplified ``stick-man'' version by drawing the lines connecting adjacent joints, based on the labels.
The observations for RAMP consist of local image patches centered around each joint in these (derived) images, but without global location information. An overview of the method is shown in \cref{fig:pose_method}.

We set the dimensionality of all the latent variables to $6$, with a multivariate Gaussian prior of $\mathcal{N}(\boldsymbol{0},\mathbf{I})$.
Message-passing factors between the latent variables were parametrised as fully-connected neural networks with one hidden layer of 64 units.
The observation-to-latent recognition factors were parameterised as convolutional neural networks with three convolutional layers.
The first layer employed 32 filters with a 3×3 kernel and a stride of 1, the second layer used 64 filters with a 3×3 kernel and a stride of 2, and the third layer used 32 filters with a 3×3 kernel and a stride of 2. 
The model was trained on minibatches of size 200 for 3,000 iterations with a learning rate of 0.0001.

\begin{figure}[h]
    \centering
    \includegraphics[width=1\linewidth]{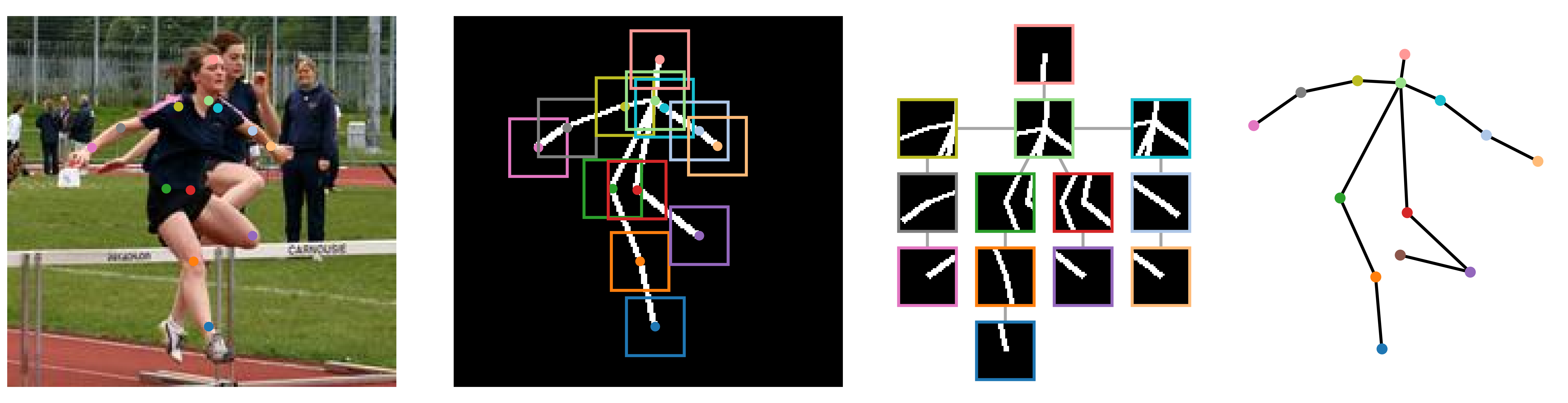}
    \caption{Pose estimation task. The LSP dataset consists of images (left) and joint positions in image space (overlaid as coloured circles for visualisation). The left ankle joint has no position label as it is occluded in this image. Each image was converted to a stick-figure representation (centre left) by drawing 
    % white 
    lines 
    between joints. 
    % on a black background. 
    The left lower leg has not been drawn as the left ankle joint position is unknown. Patches of images (32 $\times$ 32) centred at each joint (coloured bounding boxes) were used as local joint observations (centre right).
    % , edges between joints in the graph shown as grey lines).
    Note that the observation at the left ankle joint is missing. The posterior means 
    % at each joint 
    inferred by RAMP were used to construct a sketch of the pose (right).}
    \label{fig:pose_method}
\end{figure}

\FloatBarrier
\section{Additional Experimental Results}\label{supp:exp-results}
%If you have additional experimental results, you may include them in the supplementary materials.
\subsection{Linear tree}\label{subsec:linear_tree_additional_results}

\begin{figure}[h]
  \centering
  \includegraphics[width=0.35\linewidth]{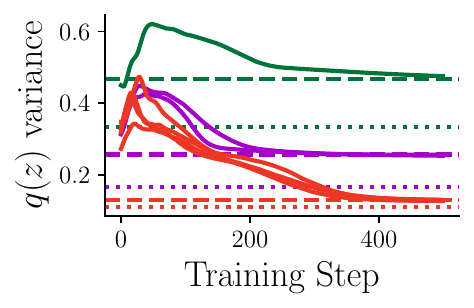}
  \caption{Posterior marginal variances (average over the dataset) inferred by RAMP, as a function of training step. Despite having no explicit model of the posterior pairwise marginals, the marginal variances learned by RAMP converge to the exact posterior variances (dashed lines), not those of the ``naive'' Mean-Field (dotted lines) achieved by restricting the posterior to factorise over the latents, indicating that correlations among latent variables are implicitly encoded in the model.}
  \label{fig:linear_tree_covs_ramp}
\end{figure}

\begin{figure}[h]
    \centering
        \hspace{.1\linewidth}BBVI\hspace{.35\linewidth}RAMP\hfill\\
        \includegraphics[width=0.35\linewidth]{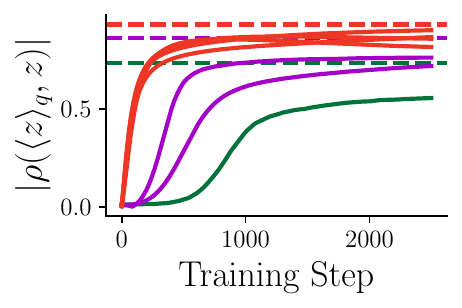}\hspace{2em}%
        \includegraphics[width=0.35\linewidth]{figures/tree_example_linear_corrs.pdf}
    \caption{Pearson correlation between inferred posterior means and true latent in the linear tree model as a function of training step for the BBVI (left) and RAMP (right; panel reproduced from \cref{fig:hierarchical_models} for reference). Dashed lines denote optimal values of the exact posterior. Note the differences in training time axes, as well as the converged performance level, between BBVI and RAMP.}
    \label{fig:bbvi_learning_curve_linear_tree}
\end{figure}

\begin{figure}[h]
    \centering
    \includegraphics[width=0.5\linewidth]{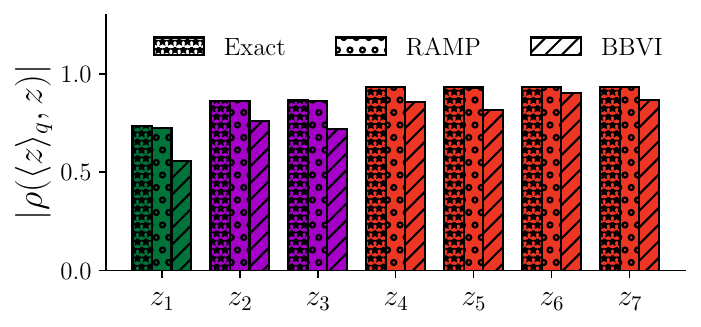}
    \caption{Pearson correlation between inferred posterior means (at the end of training) and true latents for RAMP, BBVI, and Exact inference in the linear tree model.}
    \label{fig:linear_tree_corrZ_BBVI}
\end{figure}

\FloatBarrier
\subsection{Nonlinear tree}\label{subsec:nonlinear_tree_additional_results}

The flexible recognition of RAMP enables it to discover an appropriate representational scheme for the latent variables. 
In this case, the model learns to reliably represent and infer the reciprocals of the latents, $\frac{1}{z_k}$. RAMP is able to learn a better model of the latents compared to Black Box methods that depend on an explicitly parametrised generative model, and also outperforms a supervised baseline that assumes access to the covariance matrix of the true model (\cref{fig:tree_nonlinear}).

\begin{figure}[h]
    \centering
    \includegraphics[width=.50\linewidth]{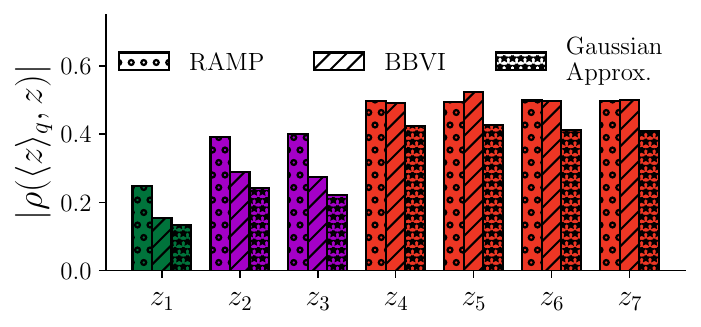}
    
    \caption{
    Hierarchical nonlinear tree-structured model. Bars show Spearman correlation between inferred means and true latents (colours indicate variable depth as in \cref{fig:hierarchical_models}).
    RAMP learns to infer latents more accurately than either BBVI, or a Gaussian approximation based on the full covariance matrix (requiring knowledge of the true latents), particularly for latents ``further away'' from the observations.
    }
    \label{fig:tree_nonlinear}
\end{figure}

The sampling process of the nonlinear tree model (\cref{subsec:nonlinear_tree_exp_details}) suggests that a natural solution for the problem would be to represent $z^{-1}$ (instead of $z$). Note that at least in the last layer of latent variables, the observations are indeed linear-Gaussian in $z^{-1}$. We hypothesized that the flexible recognition scheme used in RAMP might pick up this representation. Indeed, as shown in \cref{fig:non_linear_tree_pearson_invZ,fig:non_linear_tree_scatter_invZ} we found high \textit{Pearson} correlation between the posterior means inferred by RAMP and the (clipped) \textit{reciprocals} of the true latents.
\begin{figure}[h]
    \centering
    \includegraphics[width=0.5\linewidth]{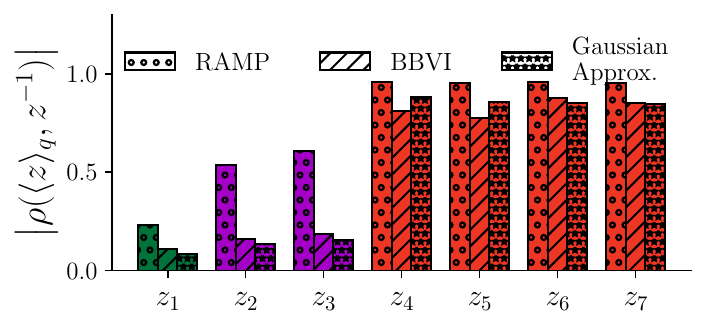}
    \caption{The Pearson correlation between inferred posterior means and the reciprocals (clipped at $[-15,15]$) of the true latents for RAMP, BBVI, and the Gaussian approximation.}
    \label{fig:non_linear_tree_pearson_invZ}
\end{figure}

\begin{figure}[h]
    \centering
    \includegraphics[width=0.47\linewidth]{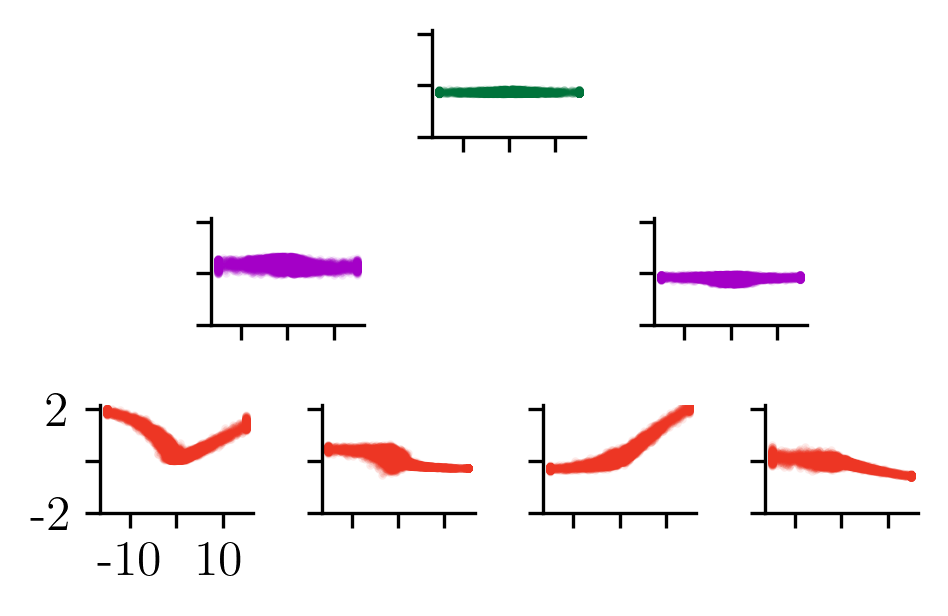}%
    \includegraphics[width=0.47\linewidth]{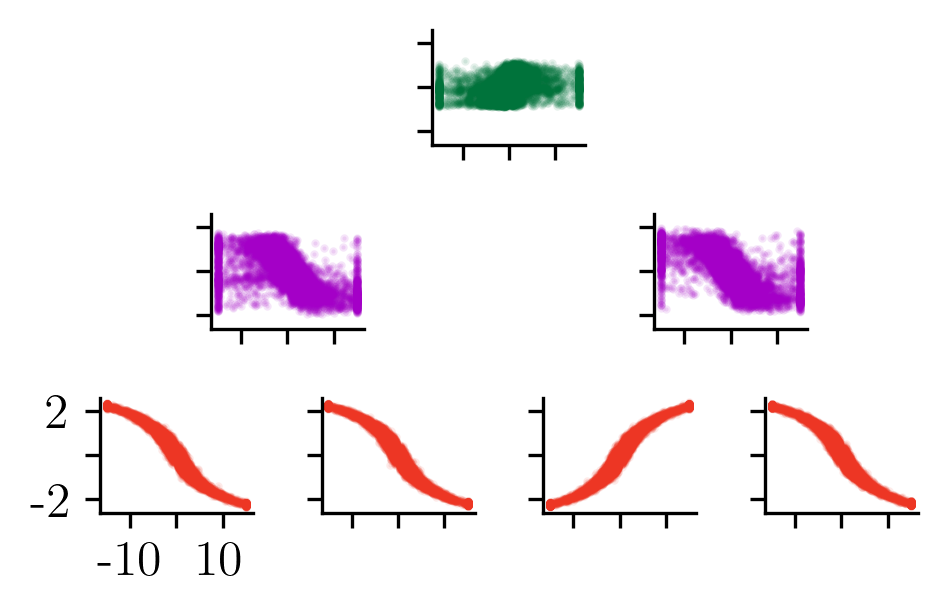}
    \caption{Scatter plot of the reciprocal of true latents (clipped at $[-15,15]$ against the posterior means inferred by RAMP before (left) and after (right) learning, in the nonlinear tree dataset.}
    \label{fig:non_linear_tree_scatter_invZ}
\end{figure}

\FloatBarrier
\subsection{Pendulum}\label{subsec:pendulum_additional_results}
To explore the robustness of RAMP's ability to learn informative latents to hyperparameters, we examined the learned latent spaces from all runs in the hyperparameter sweep (detailed in \cref{subsec:pend_exp_details}) in \cref{fig:pendulum_sweep}.

Because message-passing in RAMP is amortised through a set of neural-networks, the method is not restricted to using only Gaussian distributions for the latent variables.
\cref{fig:pendulum_beta} displays the learned latent space on the pendulum dataset from a model in which the prior and recognition factors of each $z_k$ are a product of two Beta distributions.
With this choice, we are able to impose a proper uniform prior (on the square $[0,1]\times[0,1]$) for the latents representations.

\begin{figure}[h]
    \centering
    \includegraphics[width=\linewidth]{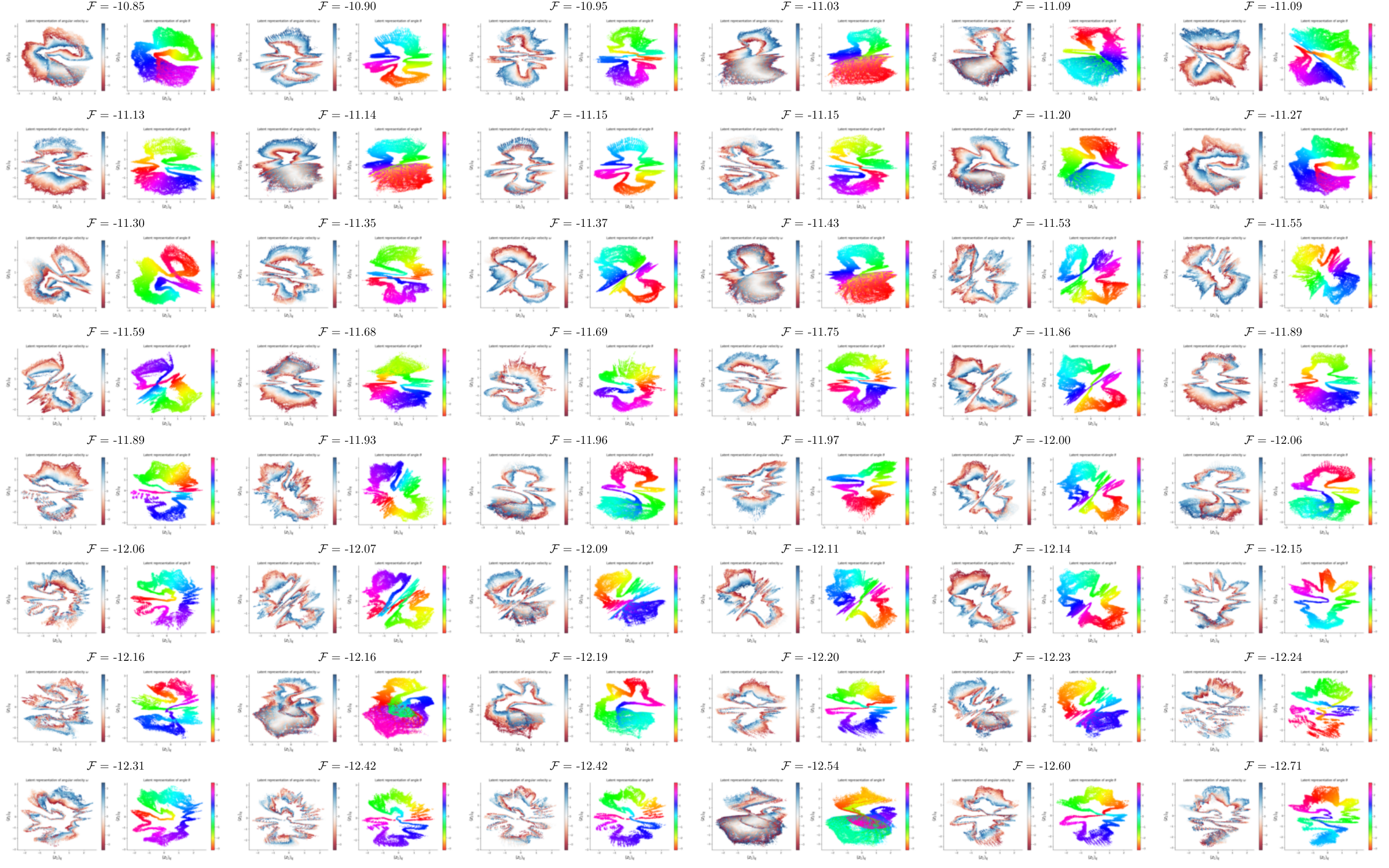}
    \caption{Latent space visualisations of all RAMP models from the pendulum hyperparameter sweep with Gaussian prior, ordered by increasing free energy. RAMP consistently and reliably learns a clean latent representation that distinguishes between different angles and angular velocities.}
    \label{fig:pendulum_sweep}
\end{figure}

\begin{figure}[h]
    \centering
    \includegraphics[width=\linewidth]{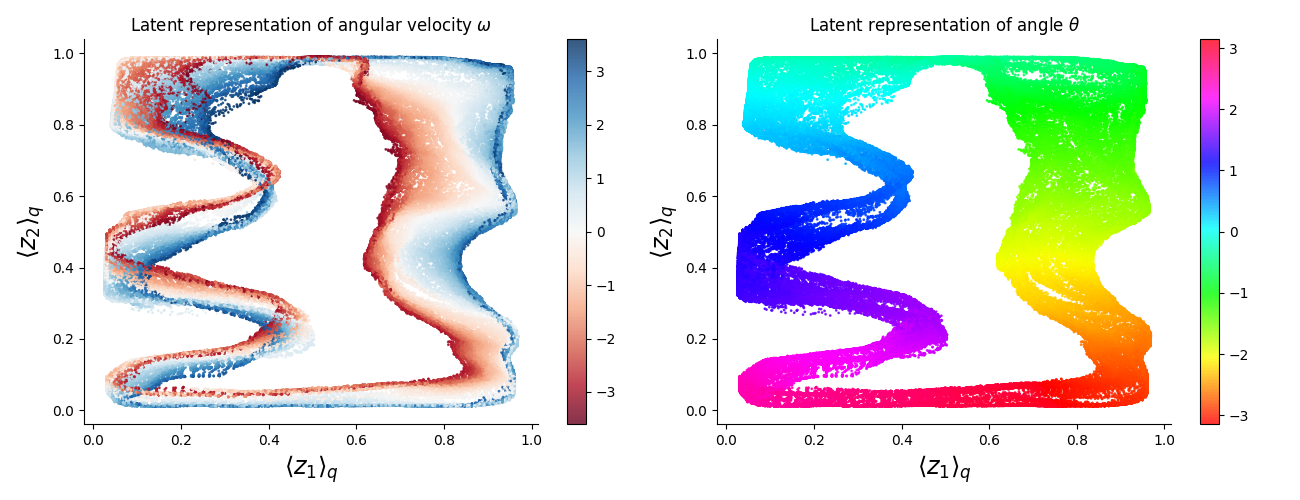}
    \caption{The learned latent space from a RAMP model with bivariate beta prior on the pendulum dataset.}
    \label{fig:pendulum_beta}
\end{figure}

\FloatBarrier
\subsection{Pose Estimation}

% For each joint and each of its neighbouring joints, 
We performed linear regression from the 6D inferred posterior mean at each joint to the 2D cosine and sine of the angle of the edge (limb) connecting the joint to a neighbouring joint.
This revealed a clear encoding of limb orientation in the inferred latent variables (Figure \ref{fig:pose_projections}).
We also performed nonlinear kernel ridge regression with an RBF kernel using scikit-learn with default hyperparameters. The resulting $R^2$ values are shown in Table \ref{tab:joint_r2} and show a strong correlation between the inferred latent variables and limb orientation.

% For each joint, we performed linear regression between the inferred posterior mean and the orientation of the limb connecting that joint to each of its neighbouring joints.

% For each joint, the
% The 6D posterior mean inferred at each joint was projected onto

% We performed linear regression from the posterior mean inferred at each joint

\begin{figure}[h]
    \centering
    \includegraphics[width=1\linewidth]{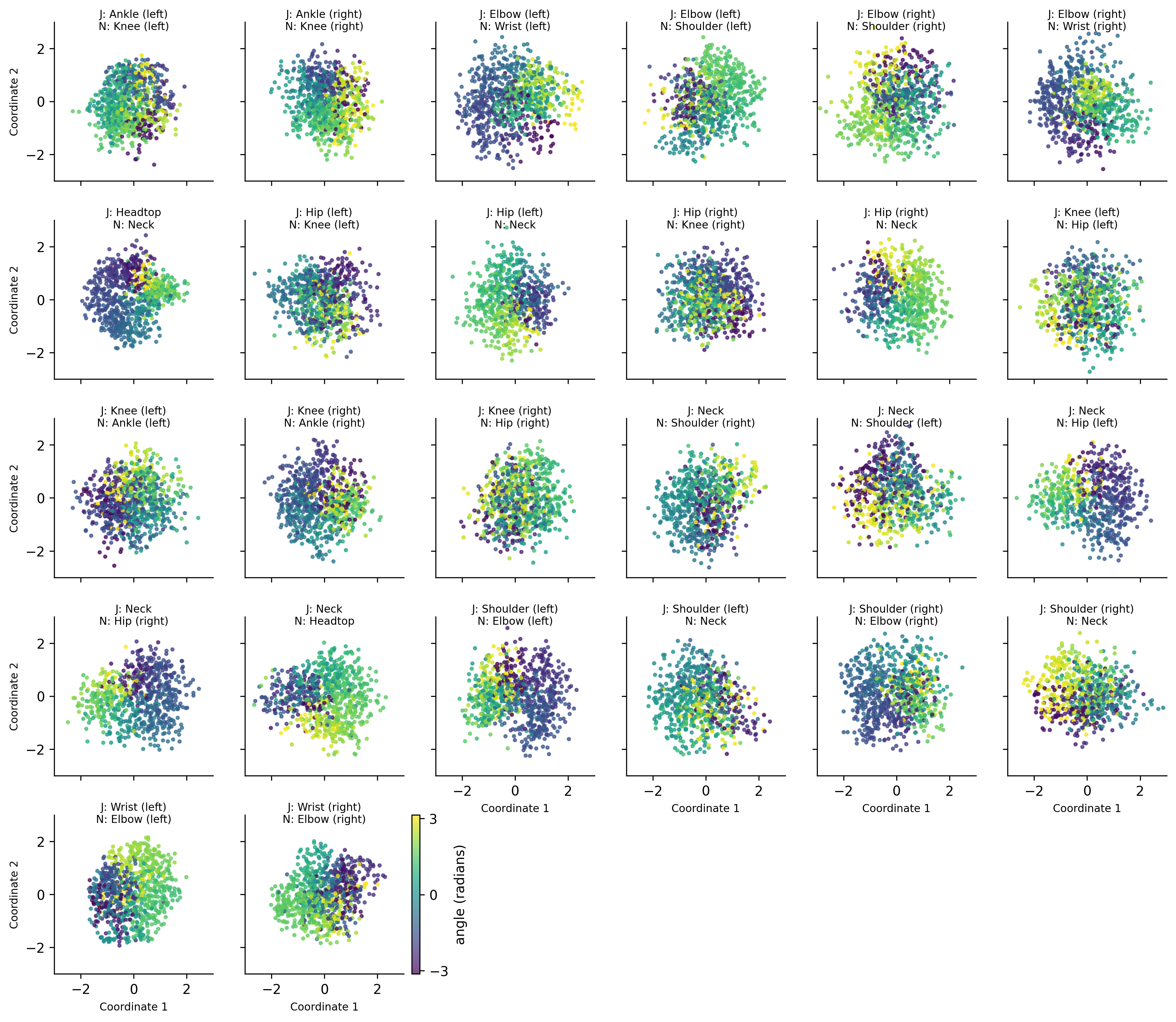}
    \caption{The 6D posterior means inferred at joint J are shown projected onto the two (orthonormalised) coefficient vectors obtained by performing linear regression from those means to the cosine and sine of the angle of the edge connecting joint J to neighbouring joint N. Each dot corresponds to a single image and is coloured based on the true angle of the edge connecting joint J to joint N.
    }
    \label{fig:pose_projections}
\end{figure}

\begin{table}[ht]
\centering
\footnotesize
\caption{$R^2$ values from kernel ridge regression predicting the orientation of edges between each joint and each of its neighbours, based on the inferred posterior means at the joint.}
\label{tab:joint_r2}
\begin{tabular}{l l c}
\hline
Joint & Neighbour & $R^2$ \\
\hline
Ankle (left) & Knee (left) & 0.83 \\
Ankle (right) & Knee (right) & 0.84 \\
Elbow (left) & Wrist (left) & 0.67 \\
Elbow (left) & Shoulder (left) & 0.56 \\
Elbow (right) & Shoulder (right) & 0.53 \\
Elbow (right) & Wrist (right) & 0.65 \\
Head (top) & Neck & 0.77 \\
Hip (left) & Knee (left) & 0.46 \\
Hip (left) & Neck & 0.68 \\
Hip (right) & Knee (right) & 0.50 \\
Hip (right) & Neck & 0.71 \\
Knee (left) & Hip (left) & 0.44 \\
Knee (left) & Ankle (left) & 0.73 \\
Knee (right) & Ankle (right) & 0.79 \\
Knee (right) & Hip (right) & 0.44 \\
Neck & Shoulder (right) & 0.40 \\
Neck & Shoulder (left) & 0.41 \\
Neck & Hip (left) & 0.79 \\
Neck & Hip (right) & 0.81 \\
Neck & Head (top) & 0.62 \\
Shoulder (left) & Elbow (left) & 0.65 \\
Shoulder (left) & Neck & 0.44 \\
Shoulder (right) & Elbow (right) & 0.56 \\
Shoulder (right) & Neck & 0.40 \\
Wrist (left) & Elbow (left) & 0.77 \\
Wrist (right) & Elbow (right) & 0.74 \\
\hline
\end{tabular}
\end{table}

\end{document}